%% file: sn-article.tex
\documentclass[pdflatex,sn-mathphys-ay]{sn-jnl}

\usepackage{graphicx}%
\usepackage{multirow}%
\usepackage{amsmath,amssymb,amsfonts}%
\usepackage{amsthm}%
\usepackage{mathrsfs}%
\usepackage[title]{appendix}%
\usepackage{xcolor}%
\usepackage{textcomp}%
\usepackage{manyfoot}%
\usepackage{booktabs}%
\usepackage{algorithm}%
\usepackage{algorithmicx}%
\usepackage{algpseudocode}%
\usepackage{listings}%

\usepackage{hyperref}
\usepackage{url}
\usepackage{mathtools}
\usepackage{multirow}
\usepackage{makecell}
\usepackage{wrapfig}
\usepackage{siunitx}
\usepackage{color}
\usepackage[]{changes} 
\usepackage{setspace}
\usepackage{amsthm}
\usepackage{ragged2e}
\usepackage{natbib}

\usepackage{hyperref}
\usepackage{marginnote}
\usepackage{lipsum}
\usepackage{makeidx}  

\definecolor{red}{rgb}{0.1,0.1,0.8}
\definecolor{green}{rgb}{0,0.4,0}

\newcommand{\change}[2]{}
\newcommand{\lchange}[2]{}

\newcommand{\changed}[3]{#3}

\newcommand{\vomega}{\boldsymbol{\omega}}

\newcommand{\vb}{\boldsymbol{b}}

\newcommand{\vx}{\boldsymbol{x}}

\newcommand{\vw}{\boldsymbol{w}}
\newcommand{\vy}{\boldsymbol{y}}

\newcommand{\mA}{\boldsymbol{A}}
\newcommand{\mC}{\boldsymbol{C}}
\newcommand{\mK}{\boldsymbol{K}}

\newcommand{\mI}{\boldsymbol{I}}
\newcommand{\mX}{\boldsymbol{X}}
\newcommand{\mM}{\boldsymbol{M}}
\newcommand{\mU}{\boldsymbol{U}}
\newcommand{\mW}{\boldsymbol{W}}
\newcommand{\mV}{\boldsymbol{V}}

\newcommand{\mZ}{\boldsymbol{Z}}

\definecolor{darkpastelpurple}{rgb}{0.59, 0.44, 0.84}

\theoremstyle{thmstyleone}%
\newtheorem{theorem}{Theorem}
\newtheorem{lemma}[theorem]{Lemma}

\newtheorem{corollary}[theorem]{Corollary}

\theoremstyle{thmstyletwo}%
\newtheorem{remark}{Remark}%

\theoremstyle{thmstylethree}%

\begin{document}


\title[Article Title]{MUSO: Achieving Exact Machine Unlearning in Over-Parameterized Regimes}


\author[1]{\fnm{Ruikai} \sur{Yang}}\email{ruikai.yang@sjtu.edu.cn}

\author[1]{\fnm{Mingzhen} \sur{He}}\email{mingzhen\_he@sjtu.edu.cn}

\author[1]{\fnm{Zhenghao} \sur{He}}\email{lstefanie@sjtu.edu.cn}

\author[1]{\fnm{Youmei} \sur{Qiu}}\email{qiuyoumei@sjtu.edu.cn}

\author*[1]{\fnm{Xiaolin} \sur{Huang}}\email{xiaolinhuang@sjtu.edu.cn}

\affil[1]{\orgdiv{Institute of Image Processing and Pattern Recognition, Department of Automation}, \orgname{Shanghai Jiao Tong University}, \orgaddress{\street{800 Dongchuan RD}, \city{Shanghai}, \postcode{200240}, \country{China}}}


\abstract{
Machine unlearning (MU) is to make a well-trained model behave as if it had never been trained on specific data. In today's over-parameterized models, dominated by neural networks, a common approach is to manually relabel data and fine-tune the well-trained model. It can approximate the MU model in the output space, but the question remains whether it can achieve exact MU, i.e., in the parameter space. We answer this question by employing random feature techniques to construct an analytical framework. Under the premise of model optimization via stochastic gradient descent, we theoretically demonstrated that over-parameterized linear models can achieve exact MU through relabeling specific data. We also extend this work to real-world nonlinear networks and propose an alternating optimization algorithm that unifies the tasks of unlearning and relabeling. The algorithm's effectiveness, confirmed through numerical experiments, highlights its superior performance in unlearning across various scenarios compared to current state-of-the-art methods, particularly excelling over similar relabeling-based MU approaches.
}




\maketitle

\section{Introduction}
The field of machine learning (ML) has undergone rapid development over the past few decades. With the ability to extract intrinsic information from provided data, ML has had a significant impact on various industries~\citep{bertsimas2022optimal, wang2023consensus, liu2024dynamic, yang2024data}. However, the increasing awareness of privacy concerns and stricter technological regulations~\citep{ginart2019making, guo2020certified}, such as the ``right to be forgotten''~\citep{hoofnagle2019european, goldman2020introduction}, have given rise to a field that runs counter to the goals of ML, referred to as the machine unlearning (MU)~\citep{cao2015towards}. MU aims to retract the influence of certain data on well-trained models, a.k.a. pre-trained models. For instance, a user may no longer wish to provide their information to companies for training ML models~\citep{ren2022prototype}, or the presence of sensitive or adverse data in the original training data could inject negative information into the model, necessitating the elimination of the impact of such information~\citep{wu2022puma}.

The most reliable method for MU is to retrain models from scratch, excluding the data intended to be forgotten or unlearned, which is considered the gold standard for MU~\citep{thudi2022necessity}. However, for many ML models, the extensive parameters and data requirements make the time and resource costs of retraining unmanageable. While several studies have achieved \textit{exact} MU, that is unlearned models completely revert to retrained models in parameter space, their primary limitations lie in their applicability to only certain simple models, such as support vector machines (SVM)~\citep{cauwenberghs2000incremental, romero2007incremental} and some low-dimensional linear models~\citep{tsai2014incremental, guo2020certified, izzo2021approximate}, or to specific unlearning scenarios~\citep{baumhauer2022machine}. Scalability also remains a significant challenge for some methods using the neural tangent kernel (NTK), which typically can only handle a few hundred data~\citep{golatkar2020forgetting}.

For today's ML models dominated by over-parameterized models, e.g., neural networks (NNs), it is challenging to achieve exact MU in the parameter space. Therefore, as long as the unlearned model matches the retrained model in metrics such as remaining accuracy, forgetting accuracy, and membership inference attacks (MIA)~\citep{shokri2017membership}, it is considered to achieve \textit{approximate} MU. \changed{M1.1}{\link{R2.1}}{More specifically, it can also be described using the (\(\epsilon\)-\(\delta\))-unlearning definition~\citep{triantafillou2024we, zhao2024makes, fan2024salun}, which quantifies the deviation between the distributions of the retrained model and the unlearned model, thereby characterizing the success of an unlearning algorithm.} A prevalent approach currently is to manually design data and then fine-tune models, termed optimization-based unlearning. Some methods focus on reinforcing the model's understanding of correct information, such as continuing to fine-tune a pre-trained model using the remaining data~\citep{warnecke2021machine}. On the other hand, certain approaches emphasize how to efficiently make the model forget information~\citep{thudi2022unrolling, shen2024camu}. This can be achieved by mislabeling the forgetting data~\citep{graves2021amnesiac,chundawat2023can}; however, these approaches tend to provide the model with incorrect information, leading to potential issues of over-forgetting~\citep{maini2023can}. To alleviate this issue, some research~\citep{he2024towards} merges the remaining data and the forgetting data while providing incorrect labels, aiming to achieve both the remembering and unlearning simultaneously. While empirically their designed merging approach shows promising results, the underlying mechanism remains unclear. This raises a key question:
\[
\emph{Can we achieve exact MU through optimization-based methods?}
\]
\changed{M1.2}{\link{R2.1} \\ \link{R2.2}}{To address this question, one challenge arises from models operating in the over-parameterized regime, where there exist countless points in the parameter space corresponding to the same output metric; they can all achieve approximate MU, but only the retrained model achieves exact MU. This also means that consistency with the retrained model in the parameter space is a sufficient but not necessary condition for consistency in the output space.} The second challenge is that the high non-linearity of NNs makes it difficult to characterize specific training dynamics, let alone attempting to manipulate the dynamics to achieve exact MU solely through data design. In this paper, we provide a positive answer to this question by resolving the first challenge and proposing an effective attempt to overcome the second difficulty.

We first use random feature (RF) techniques~\citep{rahimi2007random, liu2021random, he2024random} to construct an over-parameterized linear model and analyze the differences between various models in the unlearning task within this framework. Given the premise that the pre-training, unlearning, and retraining processes are optimized using stochastic gradient descent (SGD), we characterize the relationships in the parameter space among the models obtained from these three processes and then illustrate how the information from the remaining and forgetting data influences the differences between these models. Additionally, we prove that to achieve exact MU in over-parameterized linear models, it is only necessary to carefully adjust the labels of the forgetting data. These adjustments can be formulated as a function involving both the data and model parameters, without the need for explicit knowledge of the retrained model. 

To extend this idea to more practical models where linearity is often not applicable, we attempted to build an alternating iterative algorithm to enable the model to adaptively perform both training and relabeling tasks simultaneously.
We also apply down-sampling techniques and Woodbury matrix identity to specific variables to reduce computational and memory costs while maintaining effectiveness. This approach differs from some previous methods constrained by large-scale settings~\citep{golatkar2020forgetting, Li_2024_CVPR} and can efficiently extend to practical NN-based models. We validated our analysis with an over-parameterized linear model on the MNIST data set. By comparing the $\ell_2$ distance between the unlearned model and the retrained model, our label setting method precisely brings the difference back to zero, while other methods even amplify this metric. When tests are conducted based on ResNet18 with CIFAR data sets, our approach not only surpasses two similar relabeling-based methods but also proves competitive compared to the latest methods that focus on parameter selection. Moreover, our advantage is particularly pronounced in full-class and sub-class unlearning scenarios. The main contributions of this paper are summarized as follows:
\begin{itemize}
    \item We theoretically demonstrate that exact MU can be accomplished in the over-parameterized linear model through the optimization-based method.
    \item By aggregating the remaining data, forgetting data, and model information, we introduce an algorithm named MUSO to achieve exact MU by solely adjusting the labels of the forgetting data.
    \item Based on an alternating optimization framework with specific variable down-sampling, we extended the algorithm to NN models and large-scale data sets, enhancing its practical utility.
    \item Extensive experiments on real data sets not only validate our theory but also showcase the efficacy of MUSO, particularly demonstrating significant improvements in the MIA metric compared to other methods.
\end{itemize}

\section{Related Work}
Over the past few decades, extensive research has been conducted on the task of MU. This section will review the relevant work on exact MU and the approximate MU in NN-based models.

\textbf{Exact machine unlearning}. In the past, the process of removing the influence of partial training data from a pre-trained model was also referred to as \textit{decremental learning}, which was first explored in simple and physically meaningful models like SVM~\citep{cauwenberghs2000incremental}. Thanks to the fixed form of the solution and the model, exact MU can be achieved. Subsequently, works in~\citep{karasuyama2010multiple, tsai2014incremental, lee2019exact} have enhanced the number of data that the model can forget simultaneously. Some methods, based on the influence function instead~\citep{cook1980characterizations, koh2017understanding}, construct the responsibility of each data point towards the model to accomplish MU~\citep{guo2020certified}, but they are mostly based on low-dimensional linear models like least squares or logistic regression. To extend exact MU to more complex models, \citet{golatkar2020forgetting} proposed leveraging NTK for linear approximation of the model and performing one-step unlearning through computing the NTK matrices about the remaining data and the forgetting data. However, these types of models typically can only handle MU for data sets on the order of a few hundred instances, significantly limiting their practical applications.

\textbf{Approximate machine unlearning}. Given the vast number of parameters and diverse architectures of NN-based models today, it is challenging to exactly revert the pre-trained model to the retrained model in the parameter space. Therefore, we only aim for the unlearned model to exhibit approximate performance on certain specified metrics. In this regard, we first review optimization-based MU, which involves fine-tuning a pre-trained model with a manually adjusted data set. One approach is to fine-tune the model with forgetting data that are randomly labeled while keeping the labels of the remaining data unchanged~\citep{graves2021amnesiac}. Building upon this, \citet{fan2024salun} additionally introduce a weight saliency mask to restrict the model parameters that can be modified during fine-tuning. Following this, some work has employed a teacher-student framework, leveraging the outputs of a bad teacher as the labeling criterion for the forgetting data, and subsequently incorporating it along with the remaining data for MU~\citep{chundawat2023can, kurmanji2024towards}. As these methods may have limitations related to over-forgetting, \citet{he2024towards} aim to achieve more natural unlearning by integrating remaining data and forgetting data at the pixel level. Others perturb forgetting data randomly or adversarially to ensure that the predictions of the perturbed versions match the reference predictions, thus maintaining model performance~\citep{cha2024learning}. Additionally, there are also some optimization-free MU, such as simply performing gradient ascent on forgetting data~\citep{thudi2022unrolling} or selecting parameters most relevant to the forgetting data~\citep{foster2024fast}. However, these methods often impact the model's generalization performance, leading to less attention compared to optimization-based MU.

\section{Unlearning}
\subsection{Preliminaries}
\textbf{Notations}. The set of real numbers is written as $\mathbb{R}$. The set of integers from 1 to $N$ is written as $[N]$. We take $a$, $\boldsymbol{a}$, and $\boldsymbol{A}$ to be a scalar, a vector, and a matrix, respectively. Let $\mathrm{col}(\mA)$, $\mathrm{null}(\mA)$, $\mathrm{rank}(\mA)$, and $\mathrm{inv}(\mA)$ denote the column space, the null space, rank, and inverse matrix of $\mA$, respectively. The $n\times n$ dimensional identity matrix is written as $\mI_n$. 

\textbf{Machine Unlearning (MU)}. 
In machine learning (ML), we aim to learn a mapping function $f^\ast$ from the sample space $\mathcal{X}$ to the label space $\mathcal{Y}$. Given a data set $\mathcal{D} = \{\vx_i,y_i\}_{i=1}^N$ containing training samples $\vx_i\in\mathcal{X}\subset\mathbb{R}^d$ and labels $y_i\in\mathcal{Y}\subset\mathbb{R}$, with a pre-given task-driven loss $\ell(f(\vx), y)$, we can seek a mapping function $f$ in a suitable function space that minimizes this loss empirically to approximate $f^\ast$. Typically, this mapping function is parameterized by $\vw\in \mathbb{R}^p$. Due to factors such as protecting user data privacy, MU aims to mitigate the impact of a subset of data $\mathcal{D}_{u} = \{\vx_i^u,y_{i}^u\}_{i=1}^{N_u} \subset \mathcal{D}$ on the training of parameters $\vw$ without the need to train from scratch by using the remaining data set $\mathcal{D}_{r} = \mathcal{D}\backslash \mathcal{D}_{u} = \{\vx_i^r,y_{i}^r\}_{i=1}^{N_r}$ (which is time and resource consuming). Assuming that the ML process can be represented by algorithm $\mathcal{A}$, MU necessitates designing a unlearning mechanism $\mathcal{U}$ to derive the unlearned model  $\vw_u=\mathcal{U}(\vw_p, \mathcal{D}_r, \mathcal{D}_u)$ from the pre-trained model $\vw_p$ and to approximate the potentially existing $\vw_r = \mathcal{A}(\mathcal{D}_r)$ as closely as possible. Note again that the time required for $\mathcal{U}$ should be significantly less than that of $\mathcal{A}$. In this context, optimization-based MU as a significant branch of MU, has demonstrated excellent performance in many scenarios and tasks. It often employing manually adjusted data $\tilde{\mathcal{D}}\coloneqq \mathcal{D}_r \cup \tilde{\mathcal{D}}_u$ to fine-tune $\vw_p$ and obtain the $\vw_u$. We illustrate this kind of MU method in Figure ~\ref{fig:lct}.

\textbf{Random Features (RF) and Over-Parameterized Models}.
RF~\citep{rahimi2007random} has garnered widespread attention and application due to its ability to establish a model that strikes a balance between being easily analyzable and having practical performance~\citep{liu2022double, montanari2022interpolation, bosch2023precise, yang2024decentralized}. To achieve MU by only modifying the labels in the forgetting data set, we start with RF technology to construct an over-parameterized model. Given $\vx \in \mathbb{R}^d$, RF samples $\{\vomega_i\}_{i=1}^D, \vomega_i \in \mathbb{R}^d$ from a pre-defined probability density $p(\vomega)$ to construct an explicit mapping $z(\cdot)$ that maps the data from the original $\mathbb{R}^d$ space to a higher-dimensional space $\mathbb{R}^D$ (typically $D \gg d$). In this case, $z(\vx) = \sigma(\mW\vx)\in\mathbb{R}^D$, where $\mW=[\omega_1, \ldots, \omega_D]^\top\in\mathbb{R}^{D\times d}$ and $\sigma(\cdot)$ is a proper nonlinear activation function.
When $p(\cdot)$ follows a multivariate Gaussian distribution, this model is equivalent to a two-layer neural network (NN) initialized with random Gaussian weights, with only the output layer being optimized. In this scenario, with $Z(\mX)\coloneqq [z(\vx_1), \ldots, z(\vx_N)]\in \mathbb{R}^{D\times N}, \vy \coloneqq [y_1;\ldots;y_N]\in\mathbb{R}^{N}$ given and mean squared error (MSE) loss selected, the optimization objective would be to minimize $\mathcal{L}=\frac{1}{N}\|Z(\mX)^\top\vw - \vy\|_2^2$. For over-parameterized models, such as the popular NN-based models nowadays, this loss can always be optimized to 0 and may have infinitely many solutions. Previous work has shown that, under the optimization of stochastic gradient descent (SGD), this task is equivalent to finding a solution with the smallest $\ell_2$ distance from the initial parameters $\vw_{0}$~\citep{gunasekar2018characterizing, lin2023theory}, i.e.:
\begin{equation}\label{equ: question}
    \min_{\vw}\quad\|\vw - \vw_{0}\|_2^2,\quad\mathrm{s.t.}\ Z(\mX)^\top \vw = \vy.
\end{equation}
Next, we will present the distance in parameter space between the unlearned model $\vw_u$ and the retrained model $\vw_r$ within this framework.

\begin{figure}[tp]
\centering
\includegraphics[width=1\textwidth]{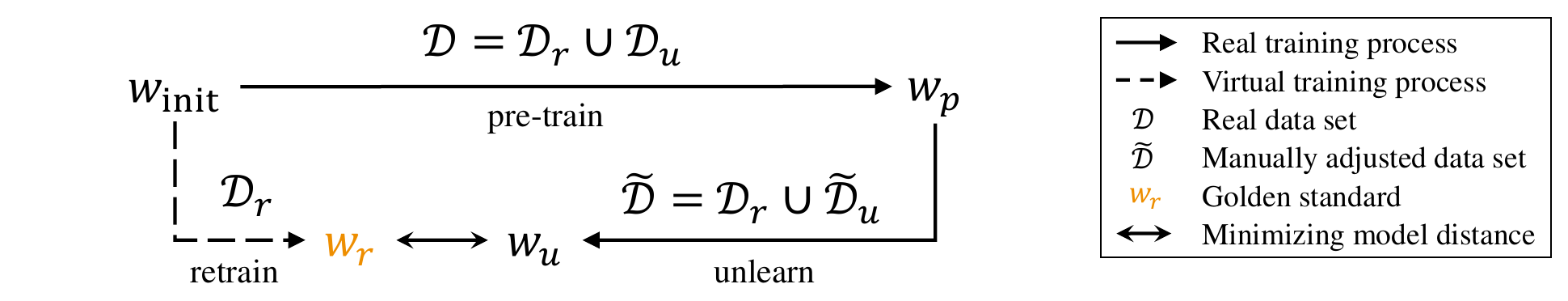}
\caption{Diagram of optimization-based MU: Let a neural network start from an initial value $\vw_{\mathrm{init}}$, go through pre-training to obtain model $\vw_p$, and then further obtain $\vw_u$ through MU. The aim is to closely align $\vw_u$ with the potentially existing $\vw_r$.} 
\label{fig:lct}
\end{figure}

\subsection{Relabeled forgetting data to achieve exact MU}\label{sec: 3.2}

\subsubsection{A review of relabeling-based MU and proof outline}

In this section, we will demonstrate that adjusting the labels of $\mathcal{D}_u$ is sufficient to align the model $\vw_u$ obtained through this process with the gold-standard model $\vw_r$ trained with the remaining data set $\mathcal{D}_r$ from $\vw_{\mathrm{init}}$. And throughout this paper, we focus on the optimization-based MU as illustrated in Figure~\ref{fig:lct}. It fine-tunes pre-trained models using a manually adjusted data set $\tilde{\mathcal{D}} \coloneqq \mathcal{D}_r \cup \tilde{\mathcal{D}}_u$, where $\mathcal{D}_{r} =  \{\vx_i^r,y_{i}^r\}_{i=1}^{N_r}$ is utilized to help the new model retain knowledge and generalization capabilities on the remaining data, while $\tilde{\mathcal{D}}_{u} = \{\tilde{\vx}_i^u,\tilde{y}_i^u\}_{i=1}^{N_u} $ guides the model on how to forget specific knowledge. 

In most cases of this MU approach, adjustments are made to the target labels $\tilde{y}_i^u$ rather than $\vx_i^u$ when constructing $\tilde{\mathcal{D}}_u$. In the Amnesiac method proposed by~\citet{graves2021amnesiac}, $\tilde{y}_i^u$ is randomly selected from any class other than the original one, i.e., $\tilde{y}_i^u = y^{\mathrm{rand}}_i \neq y^u_i$. For BadTeacher~\citep{chundawat2023can}, it requires constructing a randomly initialized model $f_{\mathrm{BT}}(\vx; \vw_{\mathrm{rand}})$ as the teacher and using the output of this model for a specific instance as $\tilde{y}_i^u = f_{\mathrm{BT}}(\vx_i^u; \vw_{\mathrm{rand}})$. Additionally, \citet{he2024towards} argue that these algorithms in order to ensure the extent of model unlearning, may excessively introduce erroneous information, leading to over-forgetting in models. Therefore, they design $\tilde{\vx}_i^u$ as a mixture of $\vx_i^u$ and another $\vx_p^r$, while simultaneously modifying the labels $\tilde{y}_i^u = y_p^r$, with the intention of achieving a more natural model unlearning process. However, these studies have not extensively explored what kind of label settings are theoretically optimal, and it remains unclear whether it is necessary to simultaneously modify both data and labels to accomplish this task. We also firmly believe that the assignment of labels should not be random but rather determined by the information present in the data. Therefore, our work attempts to bridge the gap between the theoretical aspects and practical performance of such MU methods, and provide the optimal $\tilde{y}_i^u$ when $\vw_r$ remains unknown. This goal will be achieved through the process depicted in Figure~\ref{fig: roadmap}.

\textbf{Proof Outline}. First, we leverage Lemma~\ref{lemma:1} to establish the relationships between $\vw_\mathrm{init}$, $\vw_p$, $\vw_u$, and $\vw_r$. Then, Lemma~\ref{lemma:2} introduces a block matrix inversion decomposition, allowing $\vw_p$ to serve as a bridge to formulate both $\vw_r - \vw_p$ and $\vw_p - \vw_u$, leading to Lemma~\ref{lemma:wrwpwu}. Since these two terms share common coefficients, they can be combined and simplified to derive Theorem~\ref{theo: main}: $\vw_r - \vw_u = \mC(\tilde{\vy}_u - \vb)$. A natural corollary (Corollary~\ref{coro: main}) ensures that $\vw_u = \vw_r$ when $\tilde{\vy}_u = \vb$. It is important to note that the matrices $\mC$ and $\vb$ depend solely on $\mathcal{D}_r$, $\mathcal{D}_u$, $\vw_{\mathrm{init}}$, and $\vw_p$.

\begin{figure}[tp]
\centering
\includegraphics[width=1\textwidth]{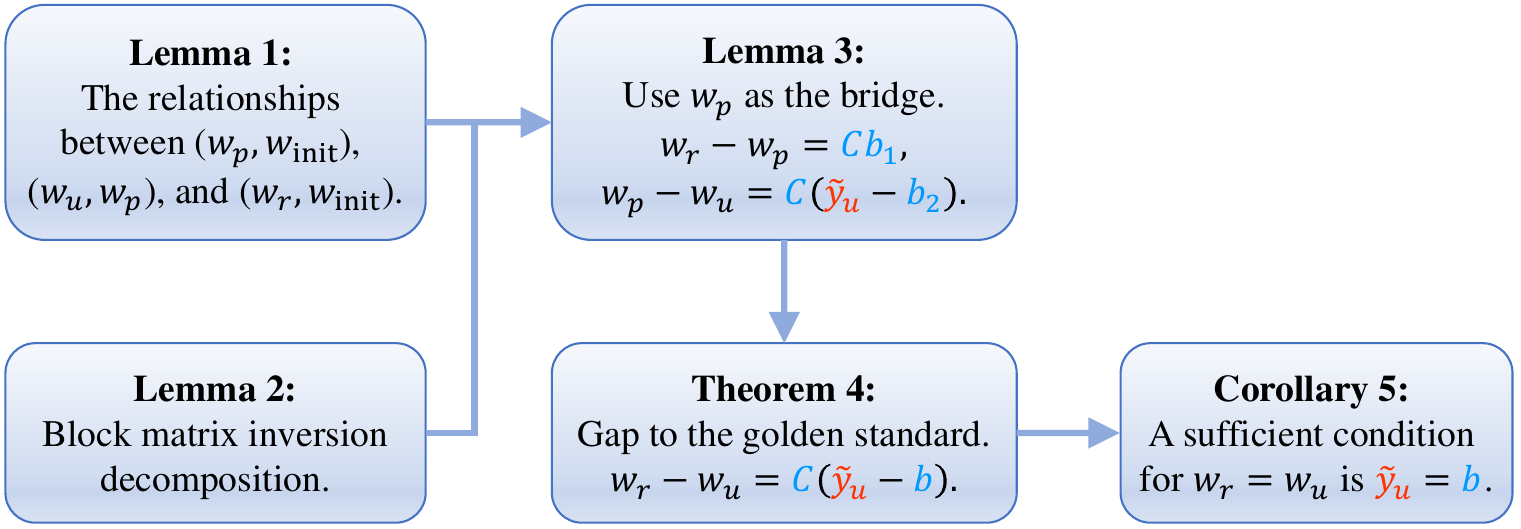}
\caption{The roadmap of proofs. The black vectors represent the model parameters, while the blue matrices and vectors are defined by the remaining data, forgetting data, $\vw_{\mathrm{init}}$ and $\vw_p$. The red $\tilde{\vy}_u$ denotes labels that can be manually adjusted.} 
\label{fig: roadmap}
\end{figure}

\subsubsection{Main results}\label{sec: 3.2.2}

To maintain generality and alleviate the complexity of analysis, we assume that all training processes in Figure~\ref{fig:lct} operate under SGD, one of the most commonly employed strategies for training NNs. In this scenario, for an over-parameterized linear model $f(\vx)=z(\vx)^\top \vw$, its optimization objective is equivalent to the problem~\eqref{equ: question}, allowing us to determine the form of the model's converged parameters using the following lemma.
\begin{lemma} \citep{lin2023theory}
    For the optimization problem~\eqref{equ: question}, the closed-form solution exists in the over-parameterized regime, which is given by
    \begin{equation}
        \vw = \vw_0 + Z(\mX)\left[Z(\mX)^\top Z(\mX)\right]^{-1}\left[\vy - Z(\mX)^\top \vw_0\right].
    \end{equation}
    \label{lemma:1}
\end{lemma}
From this, we can derive the expressions for the pre-trained model $\vw_p$, the unlearned model $\vw_u$, and the retrained model $\vw_r$. For brevity, we represent the data set as matrices $\tilde{\mathcal{D}}=\mathcal{D}_r \cup \tilde{\mathcal{D}}_u = \{\mX_r, \vy_r\}\cup \{\mX_u, \tilde{\vy}_u\}$, and let $\mX_a = [\mX_r\ \mX_u]\in\mathbb{R}^{d\times N}$, $\vy_a = [\vy_r;\vy_u]\in\mathbb{R}^{N}$, $\mZ_a = [\mZ_r\ \mZ_u] = [Z(\mX_r)\ Z(\mX_u)]\in\mathbb{R}^{D\times N}$, and $\tilde{\vy}_a=[\vy_r; \tilde{\vy}_u]\in\mathbb{R}^{N}$. By utilizing Lemma~\ref{lemma:1}, we can obtain:
\begin{subequations}
    \begin{align}
    \vw_p &= \vw_{\mathrm{init}}+\mZ_a\left(\mZ_a^\top\mZ_a\right)^{-1}\left(\vy_a - \mZ_a^\top\vw_{\mathrm{init}}\right),\label{equ:wp}\\
    \vw_u &= \vw_{p}+\mZ_a\left(\mZ_a^\top\mZ_a\right)^{-1}\left(\tilde{\vy}_a - \mZ_a^\top\vw_{p}\right),\label{equ:wu}\\
    \vw_r &= \vw_{\mathrm{init}}+\mZ_r\left(\mZ_r^\top\mZ_r\right)^{-1}\left(\vy_r - \mZ_r^\top\vw_{\mathrm{init}}\right)\label{equ:wr}.
    \end{align}
    \label{equ:wpwuwr}
\end{subequations}

\begin{lemma} \citep{zhang2006schur} Suppose $\mK_{aa} = \mZ_a^\top\mZ_a$, $\mK_{rr} = \mZ_r^\top\mZ_r$, $\mK_{ru} = \mZ_r^\top\mZ_u$ and $\mK_{uu} = \mZ_u^\top\mZ_u$, we have
    \begin{equation}
        \mK_{aa}^{-1}=
        {\begin{bmatrix}
            \mK_{rr} & \mK_{ru} \\
            \mK_{ur} & \mK_{uu}
        \end{bmatrix}}^{-1}
        =
        \begin{bmatrix}
            \mK_{rr}^{-1}+\mK_{rr}^{-1}\mK_{ru}\mM\mK_{ur}\mK_{rr}^{-1} & -\mK_{rr}^{-1}\mK_{ru}\mM \\
            -\mM\mK_{ur}\mK_{rr}^{-1} & \mM
        \end{bmatrix},
        \label{equ: lemma2}
    \end{equation}
    where $\mM \coloneqq \left(\mK_{uu} - \mK_{ur}\mK_{rr}^{-1}\mK_{ru}\right)^{-1}$.
    \label{lemma:2}
\end{lemma}
For the equations in~\eqref{equ:wpwuwr}, the primary challenge in establishing the relationship between $\vw_r$ and $\vw_u$ is how to handle the inversions of matrices within them. Thanks to Lemma~\ref{lemma:2}, we are able to decompose each part of the inverse of a block matrix into functions related to sub-block matrices. And we utilize $\vw_p$ as a bridge to initially investigate the distances between $\vw_r$ and $\vw_p$, and between $\vw_p$ and $\vw_u$, which is formulated as the following lemma.

\begin{lemma} 
    The relationship between $\vw_r$ and $\vw_u$ can be decomposed using $\vw_p$ into two terms, each represented by the same matrix multiplied by different combination coefficients.
    \begin{subequations}
        \begin{align}
            \vw_r - \vw_p &= \left(\mI_{D} - \boldsymbol{\Pi}_r\right)\mZ_u\mM\left(\mK_{ur}\mK_{rr}^{-1}\left(\vy_r - \mZ_r^\top\vw_{\mathrm{init}}\right)+\mZ_u^\top\vw_{\mathrm{init}}- \vy_u\right), \label{equ: wr-wp}\\
            \vw_p - \vw_u &= \left(\mI_{D} - \boldsymbol{\Pi}_r\right)\mZ_u\mM\left(\mZ_u^\top\vw_{p} - \tilde{\vy}_u \right),  \label{equ: wp-wu}    
        \end{align}
        \label{equ: lemma3}
    \end{subequations}
    where $\boldsymbol{\Pi}_r\coloneqq\mZ_r(\mZ_r^\top \mZ_r)^{-1}\mZ_r^\top$.
    \label{lemma:wrwpwu}
\end{lemma}

\begin{proof}[Proof of Lemma~{\upshape\ref{lemma:wrwpwu}}]
    Starting from equation~\eqref{equ: lemma2}, we can decompose equation~\eqref{equ:wp} as follows:
\begin{equation}
    \begin{aligned}
        \vw_p =& \ \vw_{\mathrm{init}}+\mZ_a\left(\mZ_a^\top\mZ_a\right)^{-1}\left(\vy_a - \mZ_a^\top\vw_{\mathrm{init}}\right)\\
        =&\ \vw_{\mathrm{init}}+\left[\mZ_r\ \mZ_u\right]\begin{bmatrix}
            \mK_{rr}^{-1}+\mK_{rr}^{-1}\mK_{ru}\mM\mK_{ur}\mK_{rr}^{-1} & -\mK_{rr}^{-1}\mK_{ru}\mM \\
            -\mM\mK_{ur}\mK_{rr}^{-1} & \mM
        \end{bmatrix}\begin{bmatrix}
            \vy_r - \mZ_r^\top\vw_{\mathrm{init}} \\
            \vy_u - \mZ_u^\top\vw_{\mathrm{init}}
        \end{bmatrix}\\
        =&\ \vw_{\mathrm{init}}+\mZ_r\left(\mK_{rr}^{-1}+\mK_{rr}^{-1}\mK_{ru}\mM\mK_{ur}\mK_{rr}^{-1}\right)\left(\vy_r - \mZ_r^\top\vw_{\mathrm{init}}\right)\\
        &-\mZ_u \mM \mK_{ur}\mK_{rr}^{-1}\left(\vy_r - \mZ_r^\top\vw_{\mathrm{init}}\right)-\mZ_r \mK_{rr}^{-1} \mK_{ru} \mM \left(\vy_u - \mZ_u^\top\vw_{\mathrm{init}}\right)\\
        &+\mZ_u \mM \left(\vy_u - \mZ_u^\top\vw_{\mathrm{init}}\right).
    \end{aligned}
    \label{equ: wpwinit}
\end{equation}
Subtracting equation~\eqref{equ: wpwinit} from equation~\eqref{equ:wr}, we obtain
\begin{equation}
    \begin{aligned}
        \vw_r - \vw_p =&-\mZ_r\mK_{rr}^{-1}\mK_{ru}\mM\mK_{ur}\mK_{rr}^{-1}\left(\vy_r - \mZ_r^\top\vw_{\mathrm{init}}\right)+\mZ_u \mM \mK_{ur}\mK_{rr}^{-1}\left(\vy_r - \mZ_r^\top\vw_{\mathrm{init}}\right)\\
        &+\mZ_r \mK_{rr}^{-1} \mK_{ru} \mM \left(\vy_u - \mZ_u^\top\vw_{\mathrm{init}}\right)-\mZ_u \mM \left(\vy_u - \mZ_u^\top\vw_{\mathrm{init}}\right)\\
        =&\left(\mI_{D} - \boldsymbol{\Pi}_r\right)\mZ_u\mM\mK_{ur}\mK_{rr}^{-1}\left(\vy_r - \mZ_r^\top\vw_{\mathrm{init}}\right)\\
        &+\left(\mI_{D} - \boldsymbol{\Pi}_r\right)\mZ_u\mM\left(\mZ_u^\top\vw_{\mathrm{init}}- \vy_u\right)\\
        =&\left(\mI_{D} - \boldsymbol{\Pi}_r\right)\mZ_u\mM\left(\mK_{ur}\mK_{rr}^{-1}\left(\vy_r - \mZ_r^\top\vw_{\mathrm{init}}\right)+\mZ_u^\top\vw_{\mathrm{init}}- \vy_u\right).
    \end{aligned}
\end{equation}
Similarly, the equation~\eqref{equ:wu} can be rewritten as
\begin{equation}
    \begin{aligned}
        \vw_{p} - \vw_u =&\ \mZ_a\left(\mZ_a^\top\mZ_a\right)^{-1}\left(\mZ_a^\top\vw_{p}-\tilde{\vy}_a\right)\\
        =&\left[\mZ_r\ \mZ_u\right]\begin{bmatrix}
            \mK_{rr}^{-1}+\mK_{rr}^{-1}\mK_{ru}\mM\mK_{ur}\mK_{rr}^{-1} & -\mK_{rr}^{-1}\mK_{ru}\mM \\
            -\mM\mK_{ur}\mK_{rr}^{-1} & \mM
        \end{bmatrix}\begin{bmatrix}
            \mZ_r^\top\vw_{p} - \vy_r \\
            \mZ_u^\top\vw_{p} - \tilde{\vy}_u 
        \end{bmatrix}\\
        =&-\mZ_r\mK_{rr}^{-1}\mK_{ru}\mM\left(\mZ_u^\top\vw_{p} - \tilde{\vy}_u \right) + \mZ_u\mM\left(\mZ_u^\top\vw_{p} - \tilde{\vy}_u \right)\\
        =&\left(\mI_{D} - \boldsymbol{\Pi}_r\right)\mZ_u\mM\left(\mZ_u^\top\vw_{p} - \tilde{\vy}_u \right).
    \end{aligned}
\end{equation}
The third ``$=$'' stems from the model operating in the over-parameterized regime, leading to $\mZ_r^\top\vw_p - \vy_r = 0$. This completes the proof.
\end{proof}

\begin{remark}
    Due to $\left(\mI_{D} - \boldsymbol{\Pi}_r\right)\mZ_u$ is the orthogonal projection of $\mZ_u$ onto the left null space of $\mZ_r$~\citep{belkin2020two}, if $\mathrm{col}(\mZ_u) \subseteq \mathrm{col}(\mZ_r)$, the $\vw_r=\vw_p=\vw_u$ holds. At this point, it is evident that the network has not learned any additional information from $\mathcal{D}_u$, let alone the need to discuss the unlearning. In Appendix~\ref{appx: discussions}, we also discuss the scenario where $\tilde{\mZ}_u$ is set as a simple combination of $\mZ_r$ and $\mZ_u$.
\end{remark}

By adding the combination coefficients of the two equations in Lemma~\ref{lemma:wrwpwu} and leveraging the characteristic of over-parameterized models $\vy_u = \mZ_u^\top\vw_p$, $\vy_r = \mZ_r^\top\vw_r$, we derive the following theorem.

\begin{theorem} (Gap to the golden standard) When the optimization problem is solved by SGD, the relationship between $\vw_r$ and $\vw_u$ can be formulated as
    \begin{equation}      
        \vw_r - \vw_u =\left(\mI_{D} - \boldsymbol{\Pi}_r\right)\mZ_u\mM\left(\mZ_u^\top\boldsymbol{\Pi}_r\left(\vw_p - \vw_{\mathrm{init}}\right)+\mZ_u^\top\vw_{\mathrm{init}} - \tilde{\vy}_u\right).
        \label{equ: wr-wu}
    \end{equation}
    \label{theo: main}
\end{theorem}

Proof can be found in Appendix~\ref{appx: proof}. Observing that the former $\left(\mI_{D} - \boldsymbol{\Pi}_r\right)\mZ_u\mM$ section relies solely on the data itself, while the latter part is linked to the labels of the forgetting data, we naturally arrive at the following sufficient condition for $\vw_r - \vw_u = 0$.

\begin{corollary} When we relabel the forgetting data as
    \begin{equation}
        \tilde{\vy}_u =\mZ_u^\top(\boldsymbol{\Pi}_r(\vw_p-\vw_{\mathrm{init}})+\vw_{\mathrm{init}}),
        \label{equ: setyu}
    \end{equation}
    the unlearned model will achieve exact MU, i.e., $\vw_u=\vw_r$.
    \label{coro: main}
\end{corollary}
The above corollary describes the optimal label setting of $\tilde{\vy}_u$, which can be viewed as a function of both model information $\vw_\mathrm{init}$, $\vw_p$, and data information $\mX_u$, $\mX_r$. Unlike the previous relabeling methods that essentially involve setting random labels~\citep{graves2021amnesiac, chundawat2023can}, thus introducing excessive incorrect information leading to over-forgetting risks, our approach provides the optimal relabeling strategy, which is also numerically confirmed in Section~\ref{sec: explinear}. 
Notably, although the model operates in an over-parameterized regime, the key lies in the fact that within SGD's optimization framework, the solutions at each process are essentially minimum norm solutions, enabling the achievement of a unique model through data relabeling and fine-tuning.

Thus, for the two main challenges in constructing the relationship between retrained models and unlearned models in modern MU: over-parameterization and non-linearity, we rigorously resolve the former with theoretical results. For the latter, we propose an approximation method in the subsequent section. \changed{M3.1}{\link{R2.2}}{Although in this scenario, the unlearning performance in the output space is not entirely equivalent to that in the parameter space, since MU is based on fine-tuning a pre-trained model, the changes in model parameters are not drastic. Within this small range, differences in the output and parameter space may still be correlated to some extent. The outcomes in Section~\ref{sec: expNN} also demonstrate its excellent performance.}

\subsection{Implementation in NN-based methods}
For the over-parameterized linear model mentioned above, we can extend it to NN-based models. We can consider the first $L-1$ layers of an $L$-layer NN as feature transformations on the input, acting as a mapping function $Z(\cdot)$. In this setup, the final fully connected layer plays the role of $\vw$. It is worth noting that for multi-class classification problems, we only need to modify the variables to $\vw\in\mathbb{R}^{D\times C}, \vy\in\mathbb{R}^{N\times C}$, where $C$ represents the number of classes. We can then use equation~(\ref{equ: setyu}) to assign labels to forgetting data. However, in practical implementation, this extension will mainly face two challenges. 

The first one arises from computing the matrix $\boldsymbol{\Pi}_r=\mZ_r(\mZ_r^\top \mZ_r)^{-1}\mZ_r^\top$. While computing the network's forward $Z(\mX_{\mathrm{batch}})$ for each batch of data is feasible, storing the entire $\mZ_r\in\mathbb{R}^{D\times N_r}$ and calculating $(\mZ_r^\top \mZ_r)^{-1}\in\mathbb{R}^{N_r\times N_r}$ becomes impractical when handling a large amount of remaining data. These processes incur storage costs of $\mathcal{O}(N_rD)$ and time costs of $\mathcal{O}(N_r^2D + N_r^3)$, respectively. 
Fortunately, the physical meaning of $\boldsymbol{\Pi}_r\vx$ is to project the vector $\vx$ onto the column space of $\mZ_r$. Therefore, we can attempt to approximate the actual $\mathrm{col}(\mZ_r)$ using the column space spanned by a subset of samples $\mZ_{r,s} = Z(\mX_{r,\mathrm{sub}})\in\mathbb{R}^{D\times N_{r,s}}$. A natural and simple approach is to randomly sample a portion of the remaining data. When $\mathrm{rank}(\mZ_r)\leqslant D\ll N_r$, there exists many linearly dependent columns in $\mZ_r$. Thus, by selecting a number of columns larger than the matrix rank, we have a high probability of preserving the matrix's linearly independent directions, effectively approximating the column space.
Additionally, there are some works that have proposed carefully designed methods for column selection~\citep{drineas2008relative, mahoney2009cur, halko2011finding}. These methods calculate leverage scores using the $k$ right singular vectors of the matrix, normalize them, and treat them as sampling probabilities for each column. Given that leverage scores reflect the relative importance of each column to the matrix operation, selecting a subset of data in this manner can lead to improved performance and stronger theoretical guarantees. Details of such \texttt{ColumnSelect} algorithms can be found in Appendix~\ref{appx: colselect}. Although sampling a subset of the data set can alleviate memory constraints, the $\mathcal{O}(N_{r,s}^3)$ complexity required to compute $(\mZ_{r,s}^{\top} \mZ_{r,s})^{-1}$ remains prohibitive. A useful trick is to add a small diagonal matrix $\lambda \mI_{N_{r,s}}$, which allows the use of the following Woodbury matrix identity:
\begin{equation}
    \left(\mA+\mU\mC\mV\right)^{-1} = \mA^{-1} - \mA^{-1}\mU\left(\mC^{-1}+\mV\mA^{-1}\mU\right)^{-1}\mV\mA^{-1},
\end{equation}
where $\mA, \mU, \mC, \mV$ are conformable matrices. By substituting $\mA = \lambda \mI_{N_{r,s}}$, $\mU = \mZ_{r,s}^{\top}$, $\mC = \mI_D$, and $\mU = \mZ_{r,s}$, $\boldsymbol{\Pi}_r$ can be approximately computed as follows: 
\begin{equation}\label{equ: comp_pir}
    \hat{\boldsymbol{\Pi}}_r = \mZ_{r,s}\left(\lambda^{-1}\mI_{N_{r,s}} - \lambda^{-1}\mZ_{r,s}^\top\left(\lambda\mI_D + \mZ_{r,s}\mZ_{r,s}^\top \right)^{-1} \mZ_{r,s}\right)\mZ_{r,s}^\top.
\end{equation}
\changed{M3.2}{\link{R2.3}}{In this case, we only require the storage cost of $\mathcal{O}(N_{r,s}D)$ and the time costs of $\mathcal{O}(N_{r,s}^2D + D^3)$. Since $D \ll N_{r,s}$, this makes the approach practical for real-world applications. It is worth noting that this approach achieves better acceleration when $D$ is typically in the range of hundreds and $N_{r,s}$ is in the range of tens of thousands, where $\mathcal{O}(D^3)$ dominates. However, when the number of samples reaches millions or even tens of millions, $\mathcal{O}(N_{r,s}^2D)$ becomes the dominant term, leading to challenging scenarios. How to handle very large-scale cases more effectively will be an extension direction for this work in the future.}

The second issue arises from the assumption that the mapping function $Z(\cdot)$ remains invariant, as in the over-parameterized linear model. This assumption deviates from the actual training dynamics of NNs. In our implementation of the NN-based method, the feature mapping function is determined by the parameters of the first $L-1$ layers, which change dynamically during training. Similarly, the parameters of the fully connected layer, represented as $\vw$, also evolve throughout the training process. Consequently, a single and consistent $Z(\cdot)$ cannot adequately describe all feature mappings during training. 
To address this challenge, we adopt an alternating optimization approach to alternate between network training and label setting to resolve this contradiction.

Note that in the following, $Z^{(t)}(\cdot)$ and $\tilde{\vy}_u^{(t)}$ are used to denote the forward pass of the preceding $L-1$ layers of the network and the assigned label at time $t$, respectively.
Revisiting Lemma~\ref{lemma:wrwpwu}, the sufficient condition for $\vw_r = \vw_u$ can be expressed as:
\begin{equation}
    \tilde{\vy}_u =\mZ_u^\top(\boldsymbol{\Pi}_r(\vw_p-\vw_{\mathrm{init}})+\vw_{\mathrm{init}})-\vy_u + \mZ_u^\top\vw_p.
\end{equation}
In the linear case, the relationship $\vy_u = \mZ_u^\top \vw_p$ always holds after pre-training. However, in the implementation of NN-based methods, the dynamic nature of $Z^{(t)}(\cdot)$ (will be reflected in $\mZ_u^{(t)}$) invalidates this relationship. As a result, we need to rewrite the above equation as follows:

\begin{equation}
    \begin{aligned}
        \tilde{\vy}_u^{(t)} = \mZ_u^{(t)\top}\left(\hat{\boldsymbol{\Pi}}_r^{(t)}(\vw_p - \vw_{\mathrm{init}}) + \vw_{\mathrm{init}}\right) - \vy_u +{\mZ_u^{(t)\top}}\vw_{p},
    \end{aligned}
    \label{equ: truesetyu}
\end{equation}
where $\hat{\boldsymbol{\Pi}}_r^{(t)}$ is constructed as described earlier. In this case, to prevent the greedy alternating optimization process from introducing excessive variations in each iteration, which could lead to algorithmic instability, it is crucial to ensure that $Z^{(t)}(\cdot)$ undergoes only minimal changes before and after unlearning. Fortunately, the nature of the unlearning task, along with certain common settings, helps mitigate this issue to a large extent. First, since unlearning is essentially a fine-tuning task, it typically requires only a small number of updates compared to the extensive adjustments made during network pre-training. Second, in practical fine-tuning scenarios, it is common to freeze certain network parameters~\citep{Li_2024_CVPR, tian2024forget}, especially those in the lower layers. Additionally, in the implementation of unlearning tasks, a very small learning rate, typically 1\%-10\% of the pre-training learning rate, is often used~\citep{shen2024camu, fan2024salun, kurmanji2024towards, cheng2024remaining}. These factors collectively help to constrain parameter changes, ensuring that the feature mapping function $Z^{(t)}(\cdot)$ exhibits the desired properties. The entire MU process by setting the optimal labels for forgetting data is summarized in Algorithm~\ref{Algo}.

\begin{algorithm}[tp]
    \caption{\textbf{M}achine \textbf{U}nlearning by \textbf{S}etting \textbf{O}ptimal labels (MUSO).}
    \label{Algo}
    \begin{algorithmic}[1]
        \Require the remaining data set $\mathcal{D}_r=\{\mX_r,\vy_r\}$, the forgetting data set $\mathcal{D}_u=\{\mX_u,\vy_u\}$, the pre-trained model $Z_p(\cdot), \vw_p$, the initial model $Z_{\mathrm{init}}(\cdot), \vw_{\mathrm{init}}$, the sample ratio $R < 1$, and a small value $\lambda$.
        \Ensure the unlearned model $Z_u(\cdot),\vw_u$.
        \State Initialize $Z_u^{(0)}(\cdot)=Z_p(\cdot)$, $\vw_u^{(0)}=\vw_p$ and $t=0$.
        \Repeat 
        \State $\mX_{r,\mathrm{sub}}\in\mathbb{R}^{d\times RN}$ by randomly or  \texttt{ColumnSelect} (See in  Algorithm~\ref{Algo: CS}).
        
        \State Construct $\mZ_{u}^{(t)}$, $\mZ_{r,s}^{(t)}$ by using the first $L-1$ layers of the NN.
        \State Compute $\hat{\boldsymbol{\Pi}}_r^{(t)}$ by equation (\ref{equ: comp_pir}).
        \State Update $\tilde{\vy}_u^{(t)}$ by equation~(\ref{equ: truesetyu}).
        \State Doing gradient descent on the unlearned model with manual data set $\tilde{\mathcal{D}}=\{\mX_r,\vy_r\}\cup\{\mX_u,\tilde{\vy}_u^{(t)}\}$
        \State $t = t +1$
        \Until the algorithm converges or reaches the maximum number of iterations.
    \end{algorithmic}
\end{algorithm}

\subsection{Discussion with neural tangent kernel}
The neural tangent kernel (NTK)~\citep{jacot2018neural}, as an analysis technique that linearizes the training dynamics of NNs, has garnered significant attention in today's NN-based research. In recent years, some studies have explored the MU within the NTK framework~\citep{golatkar2020eternal, golatkar2020forgetting, golatkar2021mixed, Li_2024_CVPR}, which shares some similarities with our work. Roughly speaking, these works can be seen as considering the derivative of the output $f$ with respect to the full network parameters as the mapping function, denoted as $z(\vx)\leftarrow \nabla_{\vw} f(\vx)$. Subsequently, the parameters of the unlearned model are computed by constructing the NTK matrix between the forgetting and remaining data, i.e., $\Theta(\mX_u,\mX_r) = \nabla_{\vw} f(\mX_u)^\top\nabla_{\vw} f(\mX_r)\in\mathbb{R}^{N_uC\times N_rC}$. In this approach, $\Delta\vw = \vw_p - \vw_u$ can be obtained in a single step through closed-form solutions, eliminating the need for iterative training. Although these techniques provide robust theoretical assurances, their principal constraint is their scalability. The vast number of parameters in contemporary mainstream NN frameworks leads to extremely low computational efficiency and significantly high storage requirements when handling NTK matrices, making such MU approaches typically viable only for data sets comprising a few hundred instances. To address this constraint, the only current attempt involves freezing certain layers to reduce the number of network parameters that need adjustment~\citep{Li_2024_CVPR}. Additionally, \citet{golatkar2020forgetting} pointed out that the outputs of such linear dynamic models do not align well with real outputs, necessitating a technique called the ``trapezium trick'' for renormalizing $\Delta\vw$. In contrast, our MU method simplifies the model optimization process by solely adjusting the labels of the forgetting data. Moreover, in matrix operations related to data volume, the special properties of the projection matrix enable us to utilize sub-data sets for computation, significantly reducing the required computational and storage costs.

\section{Numerical Experiments}
In this section, we will explore various unlearning scenarios, beginning with validating our theory under the over-parameterized linear model and then assessing the performance of our proposed MUSO in a real NN model.
\subsection{Experimental Settings}\label{sec: expsetting}
\textbf{Data sets and unlearning scenarios.} For the over-parameterized linear model, following the work of~\citep{guo2020certified, liu2022double}, we will validate the correctness of our derived formulas using the MNIST data set~\citep{lecun1998gradient}. \changed{M4.1}{\link{R2.3} \\ \link{R2.4}}{For the NN model, we will utilize CIFAR-10, CIFAR-20, CIFAR-100~\citep{krizhevsky2009learning}, and TinyImageNet-200~\citep{le2015tiny} to test the forgetting capabilities of different MU algorithms in image classification tasks.}

Our testing scenarios will encompass three forgetting mechanisms: full-class unlearning, sub-class unlearning, and random unlearning, aiming to evaluate the forgetting performance of the models comprehensively.

\noindent\textbf{Metrics.} We will follow the work of~\citep{chundawat2023can, liu2024model, shen2024camu, fan2024salun, foster2024fast} to evaluate an algorithm's performance in forgetting tasks from four perspectives. 
\begin{itemize}
    \changed{M4.2}{\link{R2.4}}{\item Retain accuracy (\textbf{RA}): This metric calculates the model's accuracy on the remaining data set to assess how much the model still \textit{remembers} about the data that needs to be retained.
    \item Test accuracy (\textbf{TA}): This metric calculates the model's accuracy on the test data set. Similar to RA, this metric is used to assess how much information the model still \textit{remembers}, but it places more emphasis on how well the model maintains its generalization performance after unlearning.} 
    \item Forget accuracy (\textbf{FA}): This metric computes the model's accuracy on the forgetting data set to assess how much information the unlearned model has \textit{forgotten}.
    \item $\Delta_{\|\vw\|}$ or membership inference attack (\textbf{MIA}): Since models with different parameters may exhibit similar RA, TA, and FA, this metric aims to focus on the difference between the model itself and the gold standard $\vw_r$. For the linear model, we can directly compute their $\ell_2$-norm distance in parameter space as $\Delta_{\|\vw\|}\coloneqq\frac{1}{D}\|\vw-\vw_{r}\|_2^2$. However, this is challenging to measure directly for complex NN models, so the most prevalent approach is to use MIA to evaluate the entropy of the model's outputs. \changed{M4.3}{\link{R2.6}}{And in this paper, we employed a prediction-based MIA method, which determines the membership of a sample based on the model's prediction~\citep{yeom2018privacy,song2019privacy,song2021systematic}.} 
\end{itemize}
Note that $\Delta_{\|\vw\|}$ is the fundamental evaluation criterion for the linear model, while for the NN model, we will consider the average gap (\textbf{AvgGap}) of the unlearned model in RA, TA, FA, and MIA as the comprehensive evaluation criterion.

\noindent\textbf{Compared methods and parameters settings.} We selected four state-of-the-art (SOTA) methods for comparison with our MUSO. Two are similar to MUSO, both based on relabeling forgetting data. The other two emphasize enhancing MU effectiveness by selecting appropriate parameters.
\begin{itemize}
    \item Amnesiac~\citep{graves2021amnesiac}: Randomly setting the labels of forgetting data to other classes.
    \item BadTeacher~\citep{chundawat2023can}: Setting the labels of forgetting data to the output of a randomly initialized network.
    \item SalUn~\citep{fan2024salun}: Computing the weight saliency map to identify the most relevant weights for the forgetting data and subsequently conducts fine-tuning.
    \item SSD~\citep{foster2024fast}: Selecting parameters that are crucial for the forgetting data but less significant for the remaining data, without re-optimizing the model.
\end{itemize}
All experiments were repeated 5 times, and the implementation was carried out using Python on a docker container with an NVIDIA$^\circledR$ GeForce RTX$^\text{TM}$ 4090 GPU and a CPU with 40 cores and 256GB memory. 
Our MUSO's parameter tuning is consistent with Amnesiac and BadTeacher, with an additional sample ratio $R$. \changed{M4.4}{\link{R2.3}}{Across all experiments, we empirically recommend $R=0.2$, except for TinyImageNet-200, where we use $R=0.1$.} Additionally, we set a small value $\lambda = 10^{-6}$ to facilitate fast matrix inversion. The weight sparsity of SalUn is set within the range of [0, 1], and we select $\alpha\in[0.1, 100]$ and $\lambda\in[0.1,5]$ for SSD according to their settings in~\citep{fan2024salun,foster2024fast}. The source code is available in \url{https://github.com/Yruikk/MUSO}.

\subsection{Results for over-parameterized linear models.}\label{sec: explinear}
\changed{M4.5}{\link{R2.1} \\ \link{R2.3}}{Here, we evaluate the performance of different methods that modify forgetting data labels in over-parameterized linear models in MU to support our theoretical results in Theorem~\ref{theo: main}, and illustrate that randomly relabeling the forgetting data can sometimes lead to seemingly good performance in the output space (e.g., RA, TA, FA), but it may still significantly deviate from the implicit optimal or ``golden'' MU model when viewed from the parameter space.} As set in~\citep{liu2022double}, we selected 300 instances of the digits 3 and 7 from the MNIST data set, creating a data set of size $N=600$. We then choose $\{\vomega_i\}_{i=1}^{D=5000}$ random features such that $D \gg N$, placing the model in an over-parameterized regime, where $p(\vomega)$ is selected as multivariate Gaussian distribution $\mathcal{N}\left(0, \frac{1}{\sigma^2}\mI\right)$ and $\sigma=20$. We tested three unlearning scenarios as detailed in Section~\ref{sec: expsetting}. For the sub-class and random unlearning scenarios, we respectively selected $N_u=200$ samples to forget from digit 7 and randomly.

\begin{table*}[tb]
\caption{Comparison of MU performance with different forgetting data relabel methods in over-parameterized linear models. We report the mean of each metric, where $\Delta_{\|\vw\|}=\frac{1}{D}\|\vw-\vw_{r}\|_2^2$ is the most important, with detailed STD provided in Appendix~\ref{appx: exp}. The difference between the metrics and those of the $\vw_r$ model is marked in \textcolor{blue}{blue}, and the best performance is highlighted in \textbf{bold}.}
\label{tab:mnist_mean}
\centering
\resizebox{\linewidth}{!}{
\input{linear_mean}
}
\end{table*}

\begin{table*}[t]
\caption{Comparison of MU performance with four state-of-the-art MU algorithms in ResNet18 with CIFAR-20/100 data set. We report the mean of each metric, with detailed standard deviation provided in Appendix~\ref{appx: exp}. The difference between the metrics and those of the $\vw_r$ model is highlighted in \textcolor{blue}{blue}, with the optimal performance emphasized in \textbf{bold}.}
\label{tab:NN_mean}
\resizebox{\linewidth}{!}{
\input{NN_mean}
}
\end{table*}

In this experiment, we compared all the relabeling methods including MUSO, and reported the MU results in Table~\ref{tab:mnist_mean}. \changed{M4.6}{\link{R2.7}}{For Amnesiac and BadTeacher, we applied early stopping based on FA to mitigate the risk of forgetting. Additionally, the randomness of label settings can occasionally result in unusually long epochs during multiple experiments, causing the model to fail to unlearn effectively. We excluded these extreme cases and reported the best results.} However, these two methods that focus on guiding the model to learn incorrect information still fall short in terms of TA and FA compared to MUSO, which meticulously integrates information from both remaining and forgetting data into labels to steer the model's learning process. Moreover, the $\Delta_{\|\vw\|}$ metric, which measures how closely the updated model aligns with the retrained model, further highlights the differences between these methods. Observing that the $\Delta_{\|\vw\|}$ of MUSO has nearly dropped to zero indicates that in all the scenarios, MUSO has achieved exact MU. This demonstrates that MUSO successfully aligns the updated model parameters with those of the retrained model, avoiding unnecessary deviation. In contrast, the $\Delta_{\|\vw\|}$ of Amnesiac and BadTeacher are higher than those of the pre-trained model. Roughly speaking, this indicates that the optimization path of the model is moving in the opposite direction to the retrained model. 
Such behavior confirms that the erroneous or heuristically generated label information used by these methods cannot effectively assist the network in achieving good MU.

\subsection{Results for NN models.}\label{sec: expNN}
\changed{M4.7}{\link{R2.1}}{In this section, we take a practical perspective, using various metrics to compare our MUSO with four SOTA methods in the output space, demonstrating its excellent performance when extended to real-world non-linear NN models. First, we conduct tests with the ResNet18 model, where we use the CIFAR-100 dataset for full-class and random unlearning scenarios and the CIFAR-20 dataset for sub-class unlearning. Then, we extend similar tests to more diverse datasets and network architectures. Specifically, we validate the performance of different MU methods under random unlearning scenarios on the VGG16-BN model with the CIFAR-10 dataset and the ResNet34 model with the TinyImageNet-200 dataset.}

In the non-linear setting, directly analyzing differences in model parameters becomes challenging due to the inherent complexity of NN models. Instead, we adopt the widely used MIA as an alternative evaluation metric.
The results with the ResNet18 for the full-class unlearning are reported at the top of Table~\ref{tab:NN_mean}, where all methods exhibit similar performance in terms of TA and FA, except for BadTeacher. However, concerning MIA, our approach significantly outperforms other methods in forgetting \textit{Rocket} and \textit{Sea} classes and only slightly lags behind SalUn and Amnesiac for the \textit{Cattle} class. The middle section of Table~\ref{tab:NN_mean} displays the results of sub-class unlearning. In this configuration, BadTeacher's performance is significantly improved compared to before, but MUSO still shows a noticeable improvement over other methods in terms of AvgGap. This is primarily attributed to the advantage of our approach in MIA, a strength that is also evident in full-class unlearning. Finally, in random unlearning, the changes between the pre-trained model and the retrained model are relatively minor, with most methods performing similarly except for SSD. Nevertheless, MUSO still demonstrates a slight advantage in the AvgGap metric for 1\% random unlearning. In the context of a higher 10\% unlearning scenario, MUSO clearly outperforms the preceding three algorithms, slightly ahead of SSD. Despite SSD achieving commendable TA and MIA metrics, its unlearning effectiveness falls short in this scenario. Our approach balances model generalization and forgetting specific data, highlighting its practicality and effectiveness.

\changed{M4.8}{\link{R2.4}}{Additionally, we tested more datasets and model frameworks under random unlearning scenarios, with the results reported in Tables~\ref{tab:VGG_mean} and Table~\ref{tab:Tiny_mean}. In the first setting, our method achieved excellent results across all four metrics but ranked second in terms of AvgGap, slightly behind the SSD algorithm. It is worth noting that SSD's strong performance was mainly due to its MIA metric being very close to the retrained model, while its RA and TA metrics showed significant deviations. In the second setting, our method maintained close alignment with the retrained model in RA, TA, and FA, while achieving a smaller MIA gap than other methods, thereby outperforming them in overall performance.}

\begin{table*}[t]
\caption{Comparison of MU performance with four state-of-the-art MU algorithms in VGG16-BN with CIFAR-10 data set. We report the mean of each metric, with detailed standard deviation provided in Appendix~\ref{appx: exp}. The difference between the metrics and those of the $\vw_r$ model is highlighted in \textcolor{blue}{blue}, with the optimal performance emphasized in \textbf{bold}.}
\label{tab:VGG_mean}
\resizebox{\linewidth}{!}{
\input{VGG_mean}
}
\end{table*}

\begin{table*}[t]
\caption{Comparison of MU performance with four state-of-the-art MU algorithms in ResNet34 with TinyImageNet-200 data set. We report the mean of each metric, with detailed standard deviation provided in Appendix~\ref{appx: exp}. The difference between the metrics and those of the $\vw_r$ model is highlighted in \textcolor{blue}{blue}, with the optimal performance emphasized in \textbf{bold}.}
\label{tab:Tiny_mean}
\resizebox{\linewidth}{!}{
\input{Tiny_mean}
}
\end{table*}

\section{Conclusion}
This paper introduces an MU algorithm named MUSO, aiming to achieve exact MU by setting optimal labels for forgetting data and completing MU through optimization-based methods. Rigorous theoretical proofs demonstrate that under over-parameterized linear models, this relabeling, which integrates information from the remaining data, forgetting data, and model parameters, can enable the model to achieve exact MU. By approximating the projection matrix related to the remaining data, we reduce the time and memory costs required by the algorithm, enabling the extension of this algorithm to nonlinear NN models and large-scale data sets within an alternating optimization framework. Numerical experiments validate that our proposed MUSO achieves better and more stable unlearning effects in various scenarios. Thanks to the relabeling strategy in MUSO that can more accurately guide the model towards optimizing in the direction to the retrained model, we observe significant improvements over other SOTA methods in such metrics, particularly in full-class and sub-class unlearning scenarios. \changed{M5.1}{\link{R2.3}}{However, due to MUSO's reliance on the remaining data set for relabeling, it still faces computational challenges when applied to data scales of millions or even tens of millions. In the future, we may address this limitation by exploring block-wise matrix approximations and methods such as Nyström approximations.}

\backmatter








\section*{Declarations}

\begin{itemize}
\item Funding: The research leading to these results has received funding from the National Key Research Development Project (2023YFF1104202), the National Natural Science Foundation of China (62376155), and the Interdisciplinary Program of Shanghai Jiao Tong University (No. YG2022ZD031).
\end{itemize}

\bibliography{sn-bibliography}

\newpage

\begin{appendices}

\section{Discussions about change data and label simultaneously}\label{appx: discussions}
Assuming we combine $\mZ_u$ and a portion of $\mZ_r$ to form $\tilde{\mZ}_u$ for training, we will demonstrate below that this will not affect the results in Lemma~\ref{lemma:wrwpwu}. Recall that $\mZ_u\in\mathbb{R}^{D\times N_u}$, $\mZ_r\in\mathbb{R}^{D\times N_r}$ and typically $N_u \ll N_r$, so we can always find a sub-column of $\mZ_{r}$, denoted as $\mZ_{r,\mathrm{sub}}\in\mathbb{R}^{D\times N_u}$ to form $\tilde{\mZ}_u\coloneqq (1-c)\mZ_{r,\mathrm{sub}}+c\mZ_u$ ($0 < c < 1$). At this point, we will fine-tune $\vw_p$ using $\tilde{\mZ}_a\coloneqq[\mZ_r\ \tilde{\mZ}_u]=[\mZ_r\ \ (1-c)\mZ_{r,\mathrm{sub}}+c\mZ_u]$ and $\tilde{\vy}_a=[\vy_r;\tilde{\vy}_{u}]$. In some prior works, $\tilde{\vy}_u$ can be set as the label $\vy_{r,\mathrm{sub}}$ corresponding to $\mZ_{r,\mathrm{sub}}$ to emphasize the naturalness of the guidance information. In this scenario, we will demonstrate that the difference between $\vw_r$ and $\vw_u$ is independent of the scalar $c$.

With $\boldsymbol{\Pi}_r\coloneqq\mZ_r(\mZ_r^\top \mZ_r)^{-1}\mZ_r^\top$, we can refer to Lemma~\ref{lemma:wrwpwu} and utilize $\vw_p$ as a bridge to elucidate the relationships between $\vw_r$ and $\vw_p$, and between $\vw_p$ and $\vw_u$. For the former, we still have 
\begin{equation}
    \vw_r - \vw_p = \left(\mI_{D} - \boldsymbol{\Pi}_r\right)\mZ_u\mM\left(\mK_{ur}\mK_{rr}^{-1}\left(\vy_r - \mZ_r^\top\vw_{\mathrm{init}}\right)+\mZ_u^\top\vw_{\mathrm{init}}- \vy_u\right),
    \label{equ: appx1}
\end{equation}
which remains consistent with equation~(\ref{equ: wr-wp}). Supposing $\tilde{\mK}_{ru} = \mZ_r^\top\tilde{\mZ_u}$, $\tilde{\mK}_{uu} = \tilde{\mZ}_u^\top\tilde{\mZ}_u$, and $\tilde{\mM} = \left(\tilde{\mK}_{uu} - \tilde{\mK}_{ur}\mK_{rr}^{-1}\tilde{\mK}_{ru}\right)^{-1}$, the latter will be rewritten as
\begin{equation}
    \begin{aligned}
        &\vw_{p} - \vw_u \\
        &=\ \tilde{\mZ}_a\left(\tilde{\mZ}_a^\top\tilde{\mZ}_a\right)^{-1}\left(\tilde{\mZ}_a^\top\vw_{p}-\tilde{\vy}_a\right)\\
        &=[\mZ_r\ \tilde{\mZ}_u]\begin{bmatrix}
            \mK_{rr}^{-1}+\mK_{rr}^{-1}\tilde{\mK}_{ru}\tilde{\mM}\tilde{\mK}_{ur}\mK_{rr}^{-1} & -\mK_{rr}^{-1}\tilde{\mK}_{ru}\tilde{\mM} \\
            -\tilde{\mM}\tilde{\mK}_{ur}\mK_{rr}^{-1} & \tilde{\mM}
        \end{bmatrix}\begin{bmatrix}
            \mZ_r^\top\vw_{p} - \vy_r \\
            \tilde{\mZ}_u^\top\vw_{p} - \tilde{\vy}_{u} 
        \end{bmatrix}\\
        &=-\mZ_r\mK_{rr}^{-1}\tilde{\mK}_{ru}\tilde{\mM}\left(\tilde{\mZ}_u^\top\vw_{p} - \tilde{\vy}_u \right) + \tilde{\mZ}_u\tilde{\mM}\left(\tilde{\mZ}_u^\top\vw_{p} - \tilde{\vy}_{u} \right)\\
        &=\left(\mI_{D} - \boldsymbol{\Pi}_r\right)\tilde{\mZ}_u\tilde{\mM}\left(\tilde{\mZ}_u^\top\vw_{p} - \tilde{\vy}_{u} \right),
    \end{aligned}
    \label{equ: appx2}
\end{equation}
where the third ``$=$'' holds because $\vy_r = \mZ_r^\top\vw_p$ is valid when the model operates in the over-parameterized regime. If we want to combine equation~(\ref{equ: appx1}) and (\ref{equ: appx2}), we first need to consider the relationship between $\mM$ and $\tilde{\mM}$, which can be expressed as
\begin{equation}
    \begin{aligned}
        \mathrm{inv}(\tilde{\mM}) =&\ \tilde{\mK}_{uu} - \tilde{\mK}_{ur}\mK_{rr}^{-1}\tilde{\mK}_{ru}\\
        =&\left(\left((1-c)\mZ_{r,\mathrm{sub}}+c\mZ_u\right)^\top\left((1-c)\mZ_{r,\mathrm{sub}}+c\mZ_u\right)\right.
        \\
        &-\left.\left((1-c)\mZ_{r,\mathrm{sub}}+c\mZ_u\right)^\top\mZ_r\mK_{rr}^{-1}\mZ_r^\top\left((1-c)\mZ_{r,\mathrm{sub}}+c\mZ_u\right)\right)\\
        =&\ (1-c)^2\mZ_{r,\mathrm{sub}}^\top \mZ_{r,\mathrm{sub}} +c(1-c)\mZ_{r,\mathrm{sub}}^\top\mZ_{u}+c(1-c)\mZ_u^\top\mZ_{r,\mathrm{sub}}+ c^2\mK_{uu}\\
        &- (1-c)^2\mZ_{r,\mathrm{sub}}^\top\mZ_{r}\mK_{rr}^{-1}\mZ_r^\top\mZ_{r,\mathrm{sub}} - c(1-c)\mZ_{r,\mathrm{sub}}^\top\mZ_{r}\mK_{rr}^{-1}\mZ_r^\top\mZ_{u}\\
        &-c(1-c)\mZ_{u}^\top\mZ_{r}\mK_{rr}^{-1}\mZ_r^\top\mZ_{r,\mathrm{sub}} - c^2\mK_{ur}\mK_{rr}^{-1}\mK _{ru}\\
        =&\ c^2\mK_{uu} - c^2\mK_{ur}\mK_{rr}^{-1}\mK_{ru}\\
        =&\ c^2\mathrm{inv}(\mM).
    \end{aligned}
    \label{equ: MandM}
\end{equation}
The fourth ``$=$'' holds because of $\mathrm{col}(\mZ_{r,\mathrm{sub}})\subseteq\mathrm{col}(\mZ_r)$, and $\mZ_{r}\mK_{rr}^{-1}\mZ_r^\top$ is the projection matrix onto $\mathrm{col}(\mZ_r)$. Therefore, the projection of $\mZ_{r,\mathrm{sub}}$ onto $\mathrm{col}(\mZ_r)$ is itself, implying $\mZ_{r}\mK_{rr}^{-1}\mZ_r^\top\mZ_{r,\mathrm{sub}}=\mZ_{r,\mathrm{sub}}$ holds. The equation~(\ref{equ: MandM}) yields $\tilde{\mM} = \frac{1}{c^2}\mM$, which can then be substituted into equation~(\ref{equ: appx2}) to obtain
\begin{equation}
    \begin{aligned}
        \vw_p - \vw_u =&\left(\mI_{D} - \boldsymbol{\Pi}_r\right)\left((1-c)\mZ_{r,\mathrm{sub}}+c\mZ_u\right)\frac{1}{c^2}\mM\left(\tilde{\mZ}_u^\top\vw_{p} - \vy_{r,\mathrm{sub}} \right)\\
        =&\left(\mI_{D} - \boldsymbol{\Pi}_r\right)c\mZ_u\frac{1}{c^2}\mM\left((1-c)\mZ_{r,\mathrm{sub}}^\top\vw_{p} + c\mZ_u^\top\vw_p - \vy_{r,\mathrm{sub}} \right)\\
        =&\left(\mI_{D} - \boldsymbol{\Pi}_r\right)\mZ_u\frac{1}{c}\mM\left(c\mZ_u^\top\vw_p - c\vy_{r,\mathrm{sub}}\right)\\
        =&\left(\mI_{D} - \boldsymbol{\Pi}_r\right)\mZ_u\mM\left(\mZ_u^\top\vw_p - \vy_{r,\mathrm{sub}}\right).
    \end{aligned}
\end{equation}
Adding the above equation to equation~(\ref{equ: appx1}), we obtain
\begin{equation}
    \begin{aligned}
        &\vw_r - \vw_u\\
        &=\ \vw_r - \vw_p + \vw_p - \vw_u\\
        &=\left(\mI_{D} - \boldsymbol{\Pi}_r\right)\mZ_u\mM\left(\mK_{ur}\mK_{rr}^{-1}\left(\vy_r - \mZ_r^\top\vw_{\mathrm{init}}\right)+\mZ_u^\top\vw_{\mathrm{init}}- \vy_u+\mZ_u^\top\vw_p - \vy_{r,\mathrm{sub}}\right)\\
        &=\left(\mI_{D} - \boldsymbol{\Pi}_r\right)\mZ_u\mM\left(\mK_{ur}\mK_{rr}^{-1}\left(\vy_r - \mZ_r^\top\vw_{\mathrm{init}}\right)+\mZ_u^\top\vw_{\mathrm{init}}- \vy_{r,\mathrm{sub}}\right),
    \end{aligned}
\end{equation}
which will depend solely on $\vy_{r,\mathrm{sub}}$ and be independent of $c$.

\section{Proof of Theorem~\ref{theo: main}}\label{appx: proof}
Recall definitions of $\mK_{rr} = \mZ_r^\top\mZ_r$, $\mK_{ur} = \mZ_u^\top\mZ_r$ and $\boldsymbol{\Pi}_r\coloneqq\mZ_r(\mZ_r^\top \mZ_r)^{-1}\mZ_r^\top$. By adding the two equations in Lemma~\ref{lemma:wrwpwu}, we have
\begin{equation}
    \begin{aligned}
        &\vw_r - \vw_p \\
        & = \left(\mI_{D} - \boldsymbol{\Pi}_r\right)\mZ_u\mM\left(\mK_{ur}\mK_{rr}^{-1}\left(\vy_r - \mZ_r^\top\vw_{\mathrm{init}}\right)+\mZ_u^\top\vw_{\mathrm{init}}- \vy_u + \mZ_u^\top\vw_{p} - \tilde{\vy}_u\right) \\  
        & = \left(\mI_{D} - \boldsymbol{\Pi}_r\right)\mZ_u\mM\left(\mK_{ur}\mK_{rr}^{-1}\left(\mZ_u^\top\vw_p - \mZ_r^\top\vw_{\mathrm{init}}\right)+\mZ_u^\top\vw_{\mathrm{init}}- \tilde{\vy}_u\right) \\ 
        & = \left(\mI_{D} - \boldsymbol{\Pi}_r\right)\mZ_u\mM\left(\mZ_u^\top\boldsymbol{\Pi}_r\left(\vw_p - \vw_{\mathrm{init}}\right)+\mZ_u^\top\vw_{\mathrm{init}} - \tilde{\vy}_u\right),
    \end{aligned}
\end{equation}
which completes the proof.
\qed

\section{Algorithm for column selecting}\label{appx: colselect}
\renewcommand{\thealgorithm}{C.1}
\begin{algorithm}[hp]
    \caption{Column Select}
    \label{Algo: CS}
    \begin{algorithmic}[1]
        \Require the matrix $\mX\in\mathbb{R}^{d\times N}$, the rank parameter $k$, and the positive number $c<N$.
        \Ensure the sub-matrix $\mX_{\mathrm{sub}}\in\mathbb{R}^{d\times c}$.
        \State Obtain the top $k$ right singular vectors $\mV_k = \texttt{TruncatedSVD}(\mX,k)\in\mathbb{R}^{k\times N}$.
        \State $p_i = \frac{1}{k}\|{(\mV_k)}_{:,i}\|_2^2$, $\forall i \in [N]$.
        \State Construct $\mX_{\mathrm{sub}}$ by re-sampling $c$ columns from $\mX$ using the multinominal distribution given by the vector $(p_1,\ \ldots,\ p_N)$.
    \end{algorithmic}
\end{algorithm}

\section{More details about Table~\ref{tab:mnist_mean},~\ref{tab:NN_mean},~\ref{tab:VGG_mean} and~\ref{tab:Tiny_mean}.}\label{appx: exp}
\begin{table*}[hp]
\caption{Comparison of MU performance with different forgetting data relabel methods in over-parameterized linear models. We report the mean$\pm$std of each metric. The best performance is highlighted in \textbf{bold}.}
\label{tab:mnist_std}
\centering
\resizebox{\linewidth}{!}{
\input{linear_std}
}
\end{table*}

\begin{table*}[hp]
\caption{Comparison of MU performance with four state-of-the-art MU algorithms in ResNet18 with CIFAR-20/100 data set. We report the mean$\pm$std of each metric. The best performance is highlighted in \textbf{bold}.}
\label{tab:NN_std}
\centering
\resizebox{\linewidth}{!}{
\input{NN_std}
}
\end{table*}

\begin{table*}[hp]
\caption{Comparison of MU performance with four state-of-the-art MU algorithms in VGG16-BN with CIFAR-10 data set. We report the mean$\pm$std of each metric. The best performance is highlighted in \textbf{bold}.}
\label{tab:VGG_std}
\centering
\resizebox{\linewidth}{!}{
\input{VGG_std}
}
\end{table*}

\begin{table*}[hp]
\caption{Comparison of MU performance with four state-of-the-art MU algorithms in ResNet34 with TinyImageNet-200 data set. We report the mean$\pm$std of each metric. The best performance is highlighted in \textbf{bold}.}
\label{tab:Tiny_std}
\centering
\resizebox{\linewidth}{!}{
\input{Tiny_std}
}
\end{table*}

\end{appendices}

\end{document}

%% file: linear_mean.tex
\begin{tabular}{ccccccc}
\toprule[1.5pt]\vspace{-0.2em}
\multirow{2.2}{*}{\begin{tabular}[c]{@{}c@{}}Unlearning \\ scenario\end{tabular}} & \multirow{2.2}{*}{Metric} & \multicolumn{5}{c}{Methods} \\ \cmidrule[0.75pt]{3-7}  
 &  & \multicolumn{1}{c|}{Retrain} & \multicolumn{1}{c}{Pre-trained} & \multicolumn{1}{c}{Amnesiac} & \multicolumn{1}{c|}{BadTeacher} & \multicolumn{1}{c}{MUSO} \\ \midrule[1pt]
 \multicolumn{1}{c|}{\multirow{4}{*}{Full-class}} & \multicolumn{1}{c|}{RA} & \multicolumn{1}{c|}{100.00} & \multicolumn{1}{c}{100.00(\textcolor{blue}{0.00})} & \multicolumn{1}{c}{100.00(\textcolor{blue}{0.00})} & \multicolumn{1}{c|}{100.00(\textcolor{blue}{0.00})} & \multicolumn{1}{c}{\textbf{100.00(\textcolor{blue}{0.00})}} \\
 \multicolumn{1}{c|}{} & \multicolumn{1}{c|}{TA} & \multicolumn{1}{c|}{50.77} & \multicolumn{1}{c}{94.42(\textcolor{blue}{43.65})} & \multicolumn{1}{c}{56.00(\textcolor{blue}{5.23})} & \multicolumn{1}{c|}{56.38(\textcolor{blue}{5.61})} & \multicolumn{1}{c}{\textbf{50.87(\textcolor{blue}{0.10})}} \\
\multicolumn{1}{c|}{} & \multicolumn{1}{c|}{FA} & \multicolumn{1}{c|}{5.33} & \multicolumn{1}{c}{100.00(\textcolor{blue}{94.13})} & \multicolumn{1}{c}{6.53(\textcolor{blue}{1.20})} & \multicolumn{1}{c|}{5.83(\textcolor{blue}{0.50})} & \multicolumn{1}{c}{\textbf{5.39(\textcolor{blue}{0.06})}} \\
\multicolumn{1}{c|}{} & \multicolumn{1}{c|}{$\Delta_{\|\vw\|}$} & \multicolumn{1}{c|}{$-$} & \multicolumn{1}{c}{\textcolor{blue}{0.1027}} & \multicolumn{1}{c}{\textcolor{blue}{0.1037}} & \multicolumn{1}{c|}{\textcolor{blue}{0.1039}} & \multicolumn{1}{c}{\textbf{\textcolor{blue}{0.0001}}} \\ \midrule
\multicolumn{1}{c|}{\multirow{4}{*}{Sub-class}} & \multicolumn{1}{c|}{RA} & \multicolumn{1}{c|}{100.00} & \multicolumn{1}{c}{100.00(\textcolor{blue}{0.00})} & \multicolumn{1}{c}{100.00(\textcolor{blue}{0.00})} & \multicolumn{1}{c|}{100.00(\textcolor{blue}{0.00})} & \multicolumn{1}{c}{\textbf{100.00(\textcolor{blue}{0.00})}} \\
\multicolumn{1}{c|}{} & \multicolumn{1}{c|}{TA} & \multicolumn{1}{c|}{92.10} & \multicolumn{1}{c}{94.58(\textcolor{blue}{2.48})} & \multicolumn{1}{c}{84.63(\textcolor{blue}{7.47})} & \multicolumn{1}{c|}{83.25(\textcolor{blue}{8.85})} & \multicolumn{1}{c}{\textbf{92.05(\textcolor{blue}{0.05})}} \\
\multicolumn{1}{c|}{} & \multicolumn{1}{c|}{FA} & \multicolumn{1}{c|}{87.60} & \multicolumn{1}{c}{100.00(\textcolor{blue}{12.40})} & \multicolumn{1}{c}{88.20(\textcolor{blue}{0.60})} & \multicolumn{1}{c|}{\textbf{87.60(\textcolor{blue}{0.00})}} & \multicolumn{1}{c}{87.30(\textcolor{blue}{0.30})} \\
\multicolumn{1}{c|}{} & \multicolumn{1}{c|}{$\Delta_{\|\vw\|}$} & \multicolumn{1}{c|}{$-$} & \multicolumn{1}{c}{\textcolor{blue}{0.0554}} & \multicolumn{1}{c}{\textcolor{blue}{0.0566}} & \multicolumn{1}{c|}{\textcolor{blue}{0.0577}} & \multicolumn{1}{c}{\textbf{\textcolor{blue}{0.0001}}} \\ \midrule
\multicolumn{1}{c|}{\multirow{4}{*}{Random}} & \multicolumn{1}{c|}{RA} & \multicolumn{1}{c|}{100.00} & \multicolumn{1}{c}{100.00(\textcolor{blue}{0.00})} & \multicolumn{1}{c}{96.80(\textcolor{blue}{3.20})} & \multicolumn{1}{c|}{95.00(\textcolor{blue}{5.00})} & \multicolumn{1}{c}{\textbf{100.00(\textcolor{blue}{0.00})}} \\
\multicolumn{1}{c|}{} & \multicolumn{1}{c|}{TA} & \multicolumn{1}{c|}{92.63} & \multicolumn{1}{c}{94.06(\textcolor{blue}{1.43})} & \multicolumn{1}{c}{84.75(\textcolor{blue}{7.88})} & \multicolumn{1}{c|}{80.97(\textcolor{blue}{11.66})} & \multicolumn{1}{c}{\textbf{92.74(\textcolor{blue}{0.11})}} \\
\multicolumn{1}{c|}{} & \multicolumn{1}{c|}{FA} & \multicolumn{1}{c|}{93.90} & \multicolumn{1}{c}{100.00(\textcolor{blue}{6.10})} & \multicolumn{1}{c}{94.20(\textcolor{blue}{0.30})} & \multicolumn{1}{c|}{93.40(\textcolor{blue}{0.50})} & \multicolumn{1}{c}{\textbf{93.90(\textcolor{blue}{0.00})}} \\
\multicolumn{1}{c|}{} & \multicolumn{1}{c|}{$\Delta_{\|\vw\|}$} & \multicolumn{1}{c|}{$-$} & \multicolumn{1}{c}{\textcolor{blue}{0.0485}} & \multicolumn{1}{c}{\textcolor{blue}{0.0501}} & \multicolumn{1}{c|}{\textcolor{blue}{0.0506}} & \multicolumn{1}{c}{\textbf{\textcolor{blue}{0.0002}}} \\ \bottomrule[1.5pt]
\end{tabular}

%% file: NN_mean.tex
\begin{tabular}{ccrrrrrr}
\toprule[1.5pt]\vspace{-0.2em}
\multirow{2.2}{*}{\begin{tabular}[c]{@{}c@{}}Unlearning\\ scenario\end{tabular}} & \multirow{2.2}{*}{Metric} & \multicolumn{6}{c}{Methods}{} \\ \cmidrule[0.75pt]{3-8} 
 &  & \multicolumn{1}{c|}{Retrain} & \multicolumn{1}{c}{Amnesiac} & \multicolumn{1}{c}{BadTeacher} & \multicolumn{1}{c}{SalUn} & \multicolumn{1}{c|}{SSD} & \multicolumn{1}{c}{MUSO} \\ \midrule[1pt]
\multicolumn{1}{c|}{\multirow{5}{*}{\begin{tabular}[c]{@{}c@{}}Full-class\\ (Rocket)\end{tabular}}} & \multicolumn{1}{c|}{RA} & \multicolumn{1}{c|}{\text{99.96}} & \multicolumn{1}{c}{99.97(\textcolor{blue}{0.01})} & \multicolumn{1}{c}{99.93(\textcolor{blue}{0.03})} & \multicolumn{1}{c}{99.97(\textcolor{blue}{0.01})} & \multicolumn{1}{c|}{99.08(\textcolor{blue}{0.88})} & \multicolumn{1}{c}{\textbf{99.97(\textcolor{blue}{0.01})}} 
\\
\multicolumn{1}{c|}{} & \multicolumn{1}{c|}{TA} & \multicolumn{1}{c|}{\text{76.26}} & \multicolumn{1}{c}{76.72(\textcolor{blue}{0.46})} & \multicolumn{1}{c}{76.69(\textcolor{blue}{0.43})} & \multicolumn{1}{c}{76.66(\textcolor{blue}{0.40})} & \multicolumn{1}{c|}{76.52(\textcolor{blue}{0.26})} & \multicolumn{1}{c}{\textbf{76.23(\textcolor{blue}{0.02})}} 
\\
\multicolumn{1}{c|}{} & \multicolumn{1}{c|}{FA} & \multicolumn{1}{c|}{0.00} & \multicolumn{1}{c}{0.00(\textcolor{blue}{0.00})} & \multicolumn{1}{c}{1.20(\textcolor{blue}{1.20})} & \multicolumn{1}{c}{0.00(\textcolor{blue}{0.00})} & \multicolumn{1}{c|}{0.00(\textcolor{blue}{0.00})} & \multicolumn{1}{c}{\textbf{0.00(\textcolor{blue}{0.00})}} 
\\
\multicolumn{1}{c|}{} & \multicolumn{1}{c|}{MIA} & \multicolumn{1}{c|}{15.96} & \multicolumn{1}{c}{9.08(\textcolor{blue}{6.88})} & \multicolumn{1}{c}{0.00(\textcolor{blue}{15.96})} & \multicolumn{1}{c}{5.28(\textcolor{blue}{10.68})} & \multicolumn{1}{c|}{1.36(\textcolor{blue}{14.60})} & \multicolumn{1}{c}{\textbf{20.24(\textcolor{blue}{4.28})}} 
\\
\multicolumn{1}{c|}{} & \multicolumn{1}{c|}{AvgGap} & \multicolumn{1}{c|}{$-$}  & \multicolumn{1}{c}{\textcolor{blue}{1.84}} & \multicolumn{1}{c}{\textcolor{blue}{4.41}} & \multicolumn{1}{c}{\textcolor{blue}{2.77}} & \multicolumn{1}{c|}{\textcolor{blue}{3.94}} & \multicolumn{1}{c}{\textcolor{blue}{\textbf{1.08}}} 
\\ \midrule
\multicolumn{1}{c|}{\multirow{5}{*}{\begin{tabular}[c]{@{}c@{}}Full-class\\ (Sea)\end{tabular}}} & \multicolumn{1}{c|}{RA} & \multicolumn{1}{c|}{\text{99.94}} & \multicolumn{1}{c}{99.97(\textcolor{blue}{0.03})} & \multicolumn{1}{c}{\textbf{99.94(\textcolor{blue}{0.00})}} & \multicolumn{1}{c}{99.97(\textcolor{blue}{0.03})} & \multicolumn{1}{c|}{98.91(\textcolor{blue}{1.03})} & \multicolumn{1}{c}{99.96(\textcolor{blue}{0.02})} \\
\multicolumn{1}{c|}{} & \multicolumn{1}{c|}{TA} & \multicolumn{1}{c|}{\text{76.53}} & \multicolumn{1}{c}{76.17(\textcolor{blue}{0.36})} & \multicolumn{1}{c}{77.65(\textcolor{blue}{1.12})} & \multicolumn{1}{c}{\textbf{76.46(\textcolor{blue}{0.07})}} & \multicolumn{1}{c|}{74.54(\textcolor{blue}{1.99})} & \multicolumn{1}{c}{76.44(\textcolor{blue}{0.09})} \\
\multicolumn{1}{c|}{} & \multicolumn{1}{c|}{FA} & \multicolumn{1}{c|}{0.00} & \multicolumn{1}{c}{0.00(\textcolor{blue}{0.00})} & \multicolumn{1}{c}{32.20(\textcolor{blue}{32.20})} & \multicolumn{1}{c}{0.00(\textcolor{blue}{0.00})} & \multicolumn{1}{c|}{0.00(\textcolor{blue}{0.00})} & \multicolumn{1}{c}{\textbf{0.00(\textcolor{blue}{0.00})}} 
\\
\multicolumn{1}{c|}{} & \multicolumn{1}{c|}{MIA} & \multicolumn{1}{c|}{25.64} & \multicolumn{1}{c}{5.88(\textcolor{blue}{19.76})} & \multicolumn{1}{c}{0.00(\textcolor{blue}{25.64})} & \multicolumn{1}{c}{5.88(\textcolor{blue}{19.76})} & \multicolumn{1}{c|}{1.64(\textcolor{blue}{24.00})} & \multicolumn{1}{c}{\textbf{24.45(\textcolor{blue}{1.19})}} 
\\
\multicolumn{1}{c|}{} & \multicolumn{1}{c|}{AvgGap} & \multicolumn{1}{c|}{$-$} & \multicolumn{1}{c}{\textcolor{blue}{5.04}} & \multicolumn{1}{c}{\textcolor{blue}{14.74}} & \multicolumn{1}{c}{\textcolor{blue}{4.96}} & \multicolumn{1}{c|}{\textcolor{blue}{6.76}} & \multicolumn{1}{c}{\textcolor{blue}{\textbf{0.33}}} 
\\ \midrule
\multicolumn{1}{c|}{\multirow{5}{*}{\begin{tabular}[c]{@{}c@{}}Full-class\\ (Cattle)\end{tabular}}} & \multicolumn{1}{c|}{RA} & \multicolumn{1}{c|}{\text{99.93}} & \multicolumn{1}{c}{99.98(\textcolor{blue}{0.05})} & \multicolumn{1}{c}{\textbf{99.94(\textcolor{blue}{0.01})}} & \multicolumn{1}{c}{99.97(\textcolor{blue}{0.04})} & \multicolumn{1}{c|}{99.07(\textcolor{blue}{0.86})} & \multicolumn{1}{c}{99.97(\textcolor{blue}{0.04})} \\
\multicolumn{1}{c|}{} & \multicolumn{1}{c|}{TA} & \multicolumn{1}{c|}{\text{76.69}} & \multicolumn{1}{c}{76.97(\textcolor{blue}{0.28})} & \multicolumn{1}{c}{77.70(\textcolor{blue}{1.01})} & \multicolumn{1}{c}{76.60(\textcolor{blue}{0.09})} & \multicolumn{1}{c|}{76.63(\textcolor{blue}{0.06})} & \multicolumn{1}{c}{\textbf{76.68(\textcolor{blue}{0.01})}} \\
\multicolumn{1}{c|}{} & \multicolumn{1}{c|}{FA} & \multicolumn{1}{c|}{0.00} & \multicolumn{1}{c}{0.00(\textcolor{blue}{0.00})} & \multicolumn{1}{c}{6.80(\textcolor{blue}{6.80})} & \multicolumn{1}{c}{0.00(\textcolor{blue}{0.00})} & \multicolumn{1}{c|}{0.00(\textcolor{blue}{0.00})} & \multicolumn{1}{c}{\textbf{0.00(\textcolor{blue}{0.00})}} 
\\
\multicolumn{1}{c|}{} & \multicolumn{1}{c|}{MIA} & \multicolumn{1}{c|}{11.48} & \multicolumn{1}{c}{10.20(\textcolor{blue}{1.28})} & \multicolumn{1}{c}{0.00(\textcolor{blue}{11.48})} & \multicolumn{1}{c}{\textbf{10.60(\textcolor{blue}{0.88})}} & \multicolumn{1}{c|}{2.48(\textcolor{blue}{9.00})} & \multicolumn{1}{c}{8.67(\textcolor{blue}{2.81})} 
\\
\multicolumn{1}{c|}{} & \multicolumn{1}{c|}{AvgGap} & \multicolumn{1}{c|}{$-$} & \multicolumn{1}{c}{\textcolor{blue}{0.40}} & \multicolumn{1}{c}{\textcolor{blue}{4.82}} & \multicolumn{1}{c}{\textcolor{blue}{\textbf{0.25}}} & \multicolumn{1}{c|}{\textcolor{blue}{2.48}} & \multicolumn{1}{c}{\textcolor{blue}{0.72}} 
\\ \midrule[1.5pt]

\multicolumn{1}{c|}{\multirow{5}{*}{\begin{tabular}[c]{@{}c@{}}Sub-class\\ (Rocket)\end{tabular}}} & \multicolumn{1}{c|}{RA} & \multicolumn{1}{c|}{\text{99.96}} & \multicolumn{1}{c}{\textbf{99.97(\textcolor{blue}{0.01})}} & \multicolumn{1}{c}{99.91(\textcolor{blue}{0.05})} & \multicolumn{1}{c}{99.93(\textcolor{blue}{0.03})} & \multicolumn{1}{c|}{99.29(\textcolor{blue}{0.68})} & \multicolumn{1}{c}{99.99(\textcolor{blue}{0.03})} 
\\
\multicolumn{1}{c|}{} & \multicolumn{1}{c|}{TA} & \multicolumn{1}{c|}{\text{84.95}} & \multicolumn{1}{c}{84.55(\textcolor{blue}{0.40})} & \multicolumn{1}{c}{\textbf{84.99(\textcolor{blue}{0.04})}} & \multicolumn{1}{c}{84.58(\textcolor{blue}{0.37})} & \multicolumn{1}{c|}{84.74(\textcolor{blue}{0.21})} & \multicolumn{1}{c}{84.46(\textcolor{blue}{0.49})} 
\\
\multicolumn{1}{c|}{} & \multicolumn{1}{c|}{FA} & \multicolumn{1}{c|}{2.60} & \multicolumn{1}{c}{\textbf{2.60(\textcolor{blue}{0.00})}} & \multicolumn{1}{c}{6.20(\textcolor{blue}{3.60})} & \multicolumn{1}{c}{3.20(\textcolor{blue}{0.60})} & \multicolumn{1}{c|}{4.00(\textcolor{blue}{1.40})} & \multicolumn{1}{c}{1.00(\textcolor{blue}{1.60})} 
\\
\multicolumn{1}{c|}{} & \multicolumn{1}{c|}{MIA} & \multicolumn{1}{c|}{21.36} & \multicolumn{1}{c}{0.00(\textcolor{blue}{21.36})} & \multicolumn{1}{c}{0.00(\textcolor{blue}{21.36})} & \multicolumn{1}{c}{0.00(\textcolor{blue}{21.36})} & \multicolumn{1}{c|}{5.28(\textcolor{blue}{16.08})} & \multicolumn{1}{c}{\textbf{20.84(\textcolor{blue}{0.52})}} 
\\
\multicolumn{1}{c|}{} & \multicolumn{1}{c|}{AvgGap} & \multicolumn{1}{c|}{$-$} & \multicolumn{1}{c}{\textcolor{blue}{5.44}} & \multicolumn{1}{c}{\textcolor{blue}{6.26}} & \multicolumn{1}{c}{\textcolor{blue}{5.59}} & \multicolumn{1}{c|}{\textcolor{blue}{4.59}} & \multicolumn{1}{c}{\textcolor{blue}{\textbf{0.66}}} 
\\ \midrule
\multicolumn{1}{c|}{\multirow{5}{*}{\begin{tabular}[c]{@{}c@{}}Sub-class\\ (Sea)\end{tabular}}} & \multicolumn{1}{c|}{RA} & \multicolumn{1}{c|}{\text{99.93}} & \multicolumn{1}{c}{\textbf{99.97(\textcolor{blue}{0.04})}} & \multicolumn{1}{c}{99.37(\textcolor{blue}{0.55})} & \multicolumn{1}{c}{99.81(\textcolor{blue}{0.12})} & \multicolumn{1}{c|}{98.17(\textcolor{blue}{1.76})} & \multicolumn{1}{c}{99.60(\textcolor{blue}{0.33})} 
\\
\multicolumn{1}{c|}{} & \multicolumn{1}{c|}{TA} & \multicolumn{1}{c|}{\text{84.74}} & \multicolumn{1}{c}{86.37(\textcolor{blue}{1.63})} & \multicolumn{1}{c}{\textbf{84.82(\textcolor{blue}{0.08})}} & \multicolumn{1}{c}{84.46(\textcolor{blue}{0.28})} & \multicolumn{1}{c|}{84.51(\textcolor{blue}{0.24})} & \multicolumn{1}{c}{82.62(\textcolor{blue}{2.12})} 
\\
\multicolumn{1}{c|}{} & \multicolumn{1}{c|}{FA} & \multicolumn{1}{c|}{82.40} & \multicolumn{1}{c}{78.00(\textcolor{blue}{4.40})} & \multicolumn{1}{c}{\textbf{78.00(\textcolor{blue}{4.40})}} & \multicolumn{1}{c}{79.20(\textcolor{blue}{3.20})} & \multicolumn{1}{c|}{64.60(\textcolor{blue}{17.80})} & \multicolumn{1}{c}{88.80(\textcolor{blue}{6.40})} 
\\
\multicolumn{1}{c|}{} & \multicolumn{1}{c|}{MIA} & \multicolumn{1}{c|}{59.08} & \multicolumn{1}{c}{0.04(\textcolor{blue}{59.04})} & \multicolumn{1}{c}{0.16(\textcolor{blue}{58.92})} & \multicolumn{1}{c}{0.00(\textcolor{blue}{59.08})} & \multicolumn{1}{c|}{7.48(\textcolor{blue}{51.60})} & \multicolumn{1}{c}{\textbf{55.88(\textcolor{blue}{3.20})}} 
\\
\multicolumn{1}{c|}{} & \multicolumn{1}{c|}{AvgGap} & \multicolumn{1}{c|}{$-$} & \multicolumn{1}{c}{\textcolor{blue}{16.28}} & \multicolumn{1}{c}{\textcolor{blue}{15.99}} & \multicolumn{1}{c}{\textcolor{blue}{15.76}} & \multicolumn{1}{c|}{\textcolor{blue}{17.85}} & \multicolumn{1}{c}{\textcolor{blue}{\textbf{3.01}}} 
\\ \midrule
\multicolumn{1}{c|}{\multirow{5}{*}{\begin{tabular}[c]{@{}c@{}}Sub-class\\ (Cattle)\end{tabular}}} & \multicolumn{1}{c|}{RA} & \multicolumn{1}{c|}{\text{99.93}} & \multicolumn{1}{c}{99.97(\textcolor{blue}{0.04})} & \multicolumn{1}{c}{99.39(\textcolor{blue}{0.54})} & \multicolumn{1}{c}{\textbf{99.93(\textcolor{blue}{0.00})}} & \multicolumn{1}{c|}{99.06(\textcolor{blue}{0.87})} & \multicolumn{1}{c}{99.92(\textcolor{blue}{0.01})} 
\\
\multicolumn{1}{c|}{} & \multicolumn{1}{c|}{TA} & \multicolumn{1}{c|}{\text{84.83}} & \multicolumn{1}{c}{84.54(\textcolor{blue}{0.29})} & \multicolumn{1}{c}{85.01(\textcolor{blue}{0.18})} & \multicolumn{1}{c}{\textbf{84.72(\textcolor{blue}{0.11})}} & \multicolumn{1}{c|}{83.51(\textcolor{blue}{1.32})} & \multicolumn{1}{c}{83.95(\textcolor{blue}{0.88})} 
\\
\multicolumn{1}{c|}{} & \multicolumn{1}{c|}{FA} & \multicolumn{1}{c|}{47.80} & \multicolumn{1}{c}{38.40(\textcolor{blue}{9.40})} & \multicolumn{1}{c}{42.80(\textcolor{blue}{5.00})} & \multicolumn{1}{c}{33.80(\textcolor{blue}{14.00})} & \multicolumn{1}{c|}{16.40(\textcolor{blue}{31.40})} & \multicolumn{1}{c}{\textbf{52.40(\textcolor{blue}{4.6})}} 
\\
\multicolumn{1}{c|}{} & \multicolumn{1}{c|}{MIA} & \multicolumn{1}{c|}{26.64} & \multicolumn{1}{c}{0.00(\textcolor{blue}{26.64})} & \multicolumn{1}{c}{0.00(\textcolor{blue}{26.64})} & \multicolumn{1}{c}{0.12(\textcolor{blue}{26.52})} & \multicolumn{1}{c|}{8.80(\textcolor{blue}{17.84})} & \multicolumn{1}{c}{\textbf{25.14(\textcolor{blue}{1.50})}} 
\\
\multicolumn{1}{c|}{} & \multicolumn{1}{c|}{AvgGap} & \multicolumn{1}{c|}{$-$} & \multicolumn{1}{c}{\textcolor{blue}{9.09}} & \multicolumn{1}{c}{\textcolor{blue}{8.09}} & \multicolumn{1}{c}{\textcolor{blue}{10.16}} & \multicolumn{1}{c|}{\textcolor{blue}{12.86}} & \multicolumn{1}{c}{\textcolor{blue}{\textbf{1.75}}}
\\ \midrule[1.5pt]

\multicolumn{1}{c|}{\multirow{5}{*}{\begin{tabular}[c]{@{}c@{}}Random\\ (1\%)\end{tabular}}} & \multicolumn{1}{c|}{RA} & \multicolumn{1}{c|}{\text{99.99}} & \multicolumn{1}{c}{99.92(\textcolor{blue}{0.07})} & \multicolumn{1}{c}{99.88(\textcolor{blue}{0.11})} & \multicolumn{1}{c}{\textbf{99.93(\textcolor{blue}{0.06})}} & \multicolumn{1}{c|}{98.16(\textcolor{blue}{1.83})} & \multicolumn{1}{c}{99.93(\textcolor{blue}{0.06})} 
\\
\multicolumn{1}{c|}{} & \multicolumn{1}{c|}{TA} & \multicolumn{1}{c|}{\text{76.41}} & \multicolumn{1}{c}{76.00(\textcolor{blue}{0.41})} & \multicolumn{1}{c}{75.77(\textcolor{blue}{0.64})} & \multicolumn{1}{c}{\textbf{76.16(\textcolor{blue}{0.25})}} & \multicolumn{1}{c|}{72.48(\textcolor{blue}{3.93})} & \multicolumn{1}{c}{75.93(\textcolor{blue}{0.47})} 
\\
\multicolumn{1}{c|}{} & \multicolumn{1}{c|}{FA} & \multicolumn{1}{c|}{75.28} & \multicolumn{1}{c}{\textbf{75.68(\textcolor{blue}{0.40})}} & \multicolumn{1}{c}{76.60(\textcolor{blue}{1.32})} & \multicolumn{1}{c}{82.16(\textcolor{blue}{6.88})} & \multicolumn{1}{c|}{93.40(\textcolor{blue}{18.12})} & \multicolumn{1}{c}{76.20(\textcolor{blue}{0.92})} 
\\
\multicolumn{1}{c|}{} & \multicolumn{1}{c|}{MIA} & \multicolumn{1}{c|}{55.44} & \multicolumn{1}{c}{10.28(\textcolor{blue}{45.16})} & \multicolumn{1}{c}{10.80(\textcolor{blue}{44.64})} & \multicolumn{1}{c}{18.44(\textcolor{blue}{37.00})} & \multicolumn{1}{c|}{\textbf{79.16(\textcolor{blue}{23.72})}} & \multicolumn{1}{c}{19.72(\textcolor{blue}{35.72})} 
\\
\multicolumn{1}{c|}{} & \multicolumn{1}{c|}{AvgGap} & \multicolumn{1}{c|}{$-$} & \multicolumn{1}{c}{\textcolor{blue}{11.51}} & \multicolumn{1}{c}{\textcolor{blue}{11.68}} & \multicolumn{1}{c}{\textcolor{blue}{11.05}} & \multicolumn{1}{c|}{\textcolor{blue}{11.90}} & \multicolumn{1}{c}{\textcolor{blue}{\textbf{9.29}}} 
\\ \midrule
\multicolumn{1}{c|}{\multirow{5}{*}{\begin{tabular}[c]{@{}c@{}}Random\\ (10\%)\end{tabular}}} & \multicolumn{1}{c|}{RA} & \multicolumn{1}{c|}{\text{99.99}} & \multicolumn{1}{c}{\textbf{99.89(\textcolor{blue}{0.10})}} & \multicolumn{1}{c}{99.86(\textcolor{blue}{0.13})} & \multicolumn{1}{c}{99.83(\textcolor{blue}{0.17})} & \multicolumn{1}{c|}{96.04(\textcolor{blue}{3.95})} & \multicolumn{1}{c}{99.85(\textcolor{blue}{0.14})} 
\\
\multicolumn{1}{c|}{} & \multicolumn{1}{c|}{TA} & \multicolumn{1}{c|}{\text{75.44}} & \multicolumn{1}{c}{72.70(\textcolor{blue}{2.74})} & \multicolumn{1}{c}{\textbf{74.44(\textcolor{blue}{1.00})}} & \multicolumn{1}{c}{72.28(\textcolor{blue}{3.16})} & \multicolumn{1}{c|}{72.59(\textcolor{blue}{2.84})} & \multicolumn{1}{c}{70.04(\textcolor{blue}{5.39})} 
\\
\multicolumn{1}{c|}{} & \multicolumn{1}{c|}{FA} & \multicolumn{1}{c|}{74.68} & \multicolumn{1}{c}{75.95(\textcolor{blue}{1.27})} & \multicolumn{1}{c}{\textbf{74.86(\textcolor{blue}{0.18})}} & \multicolumn{1}{c}{76.92(\textcolor{blue}{2.24})} & \multicolumn{1}{c|}{92.88(\textcolor{blue}{18.20})} & \multicolumn{1}{c}{73.09(\textcolor{blue}{1.59})} 
\\
\multicolumn{1}{c|}{} & \multicolumn{1}{c|}{MIA} & \multicolumn{1}{c|}{54.59} & \multicolumn{1}{c}{4.02(\textcolor{blue}{50.57})} & \multicolumn{1}{c}{0.08(\textcolor{blue}{54.51})} & \multicolumn{1}{c}{2.53(\textcolor{blue}{52.06})} & \multicolumn{1}{c|}{\textbf{81.41(\textcolor{blue}{26.82})}} & \multicolumn{1}{c}{18.32(\textcolor{blue}{36.27})} 
\\
\multicolumn{1}{c|}{} & \multicolumn{1}{c|}{AvgGap} & \multicolumn{1}{c|}{$-$} & \multicolumn{1}{c}{\textcolor{blue}{13.67}} & \multicolumn{1}{c}{\textcolor{blue}{13.95}} & \multicolumn{1}{c}{\textcolor{blue}{14.41}} & \multicolumn{1}{c|}{\textcolor{blue}{12.95}} & \multicolumn{1}{c}{\textcolor{blue}{\textbf{10.85}}}
\\ \bottomrule[1.5pt]
\end{tabular}

%% file: VGG_mean.tex
\begin{tabular}{ccrrrrrr}
\toprule[1.5pt]\vspace{-0.2em}
\multirow{2.2}{*}{\begin{tabular}[c]{@{}c@{}}Unlearning\\ scenario\end{tabular}} & \multirow{2.2}{*}{Metric} & \multicolumn{6}{c}{Methods}{} \\ \cmidrule[0.75pt]{3-8} 
 &  & \multicolumn{1}{c|}{Retrain} & \multicolumn{1}{c}{Amnesiac} & \multicolumn{1}{c}{BadTeacher} & \multicolumn{1}{c}{SalUn} & \multicolumn{1}{c|}{SSD} & \multicolumn{1}{c}{MUSO} \\ \midrule[1pt]

\multicolumn{1}{c|}{\multirow{5}{*}{\begin{tabular}[c]{@{}c@{}}Random\\ (1\%)\end{tabular}}} & \multicolumn{1}{c|}{RA} & \multicolumn{1}{c|}{\text{99.97}} & \multicolumn{1}{c}{100.00(\textcolor{blue}{0.03})} & \multicolumn{1}{c}{99.95(\textcolor{blue}{0.02})} & \multicolumn{1}{c}{100.00(\textcolor{blue}{0.03})} & \multicolumn{1}{c|}{93.37(\textcolor{blue}{6.60})} & \multicolumn{1}{c}{\textbf{99.99(\textcolor{blue}{0.02})}} 
\\
\multicolumn{1}{c|}{} & \multicolumn{1}{c|}{TA} & \multicolumn{1}{c|}{\text{93.38}} & \multicolumn{1}{c}{93.00(\textcolor{blue}{0.38})} & \multicolumn{1}{c}{\textbf{93.12(\textcolor{blue}{0.26})}} & \multicolumn{1}{c}{92.71(\textcolor{blue}{0.67})} & \multicolumn{1}{c|}{84.05(\textcolor{blue}{9.33})} & \multicolumn{1}{c}{92.73(\textcolor{blue}{0.65})} 
\\
\multicolumn{1}{c|}{} & \multicolumn{1}{c|}{FA} & \multicolumn{1}{c|}{92.50} & \multicolumn{1}{c}{89.96(\textcolor{blue}{2.54})} & \multicolumn{1}{c}{92.36(\textcolor{blue}{0.14})} & \multicolumn{1}{c}{92.13(\textcolor{blue}{0.37})} & \multicolumn{1}{c|}{\textbf{92.60(\textcolor{blue}{0.10})}} & \multicolumn{1}{c}{92.20(\textcolor{blue}{0.30})} 
\\
\multicolumn{1}{c|}{} & \multicolumn{1}{c|}{MIA} & \multicolumn{1}{c|}{79.55} & \multicolumn{1}{c}{41.76(\textcolor{blue}{37.79})} & \multicolumn{1}{c}{52.80(\textcolor{blue}{26.75})} & \multicolumn{1}{c}{47.15(\textcolor{blue}{32.40})} & \multicolumn{1}{c|}{\textbf{77.50(\textcolor{blue}{2.05})}} & \multicolumn{1}{c}{55.00(\textcolor{blue}{24.55})} 
\\
\multicolumn{1}{c|}{} & \multicolumn{1}{c|}{AvgGap} & \multicolumn{1}{c|}{$-$} & \multicolumn{1}{c}{\textcolor{blue}{10.18}} & \multicolumn{1}{c}{\textcolor{blue}{6.79}} & \multicolumn{1}{c}{\textcolor{blue}{8.36}} & \multicolumn{1}{c|}{\textcolor{blue}{\textbf{4.52}}} & \multicolumn{1}{c}{\textcolor{blue}{6.38}} 
\\ \midrule

\multicolumn{1}{c|}{\multirow{5}{*}{\begin{tabular}[c]{@{}c@{}}Random\\ (10\%)\end{tabular}}} & \multicolumn{1}{c|}{RA} & \multicolumn{1}{c|}{\text{99.96}} & \multicolumn{1}{c}{100.00(\textcolor{blue}{0.04})} & \multicolumn{1}{c}{\textbf{99.97(\textcolor{blue}{0.01})}} & \multicolumn{1}{c}{99.99(\textcolor{blue}{0.03})} & \multicolumn{1}{c|}{92.83(\textcolor{blue}{7.13})} & \multicolumn{1}{c}{99.98(\textcolor{blue}{0.02})} 
\\
\multicolumn{1}{c|}{} & \multicolumn{1}{c|}{TA} & \multicolumn{1}{c|}{\text{92.89}} & \multicolumn{1}{c}{\textbf{92.90(\textcolor{blue}{0.01})}} & \multicolumn{1}{c}{92.91(\textcolor{blue}{0.02})} & \multicolumn{1}{c}{92.62(\textcolor{blue}{0.27})} & \multicolumn{1}{c|}{85.42(\textcolor{blue}{7.47})} & \multicolumn{1}{c}{91.76(\textcolor{blue}{1.13})} 
\\
\multicolumn{1}{c|}{} & \multicolumn{1}{c|}{FA} & \multicolumn{1}{c|}{93.41} & \multicolumn{1}{c}{93.72(\textcolor{blue}{0.31})} & \multicolumn{1}{c}{92.77(\textcolor{blue}{0.64})} & \multicolumn{1}{c}{93.59(\textcolor{blue}{0.18})} & \multicolumn{1}{c|}{94.89(\textcolor{blue}{1.48})} & \multicolumn{1}{c}{\textbf{93.35(\textcolor{blue}{0.06})}} 
\\
\multicolumn{1}{c|}{} & \multicolumn{1}{c|}{MIA} & \multicolumn{1}{c|}{80.65} & \multicolumn{1}{c}{35.06(\textcolor{blue}{45.59})} & \multicolumn{1}{c}{30.69(\textcolor{blue}{49.96})} & \multicolumn{1}{c}{32.73(\textcolor{blue}{47.92})} & \multicolumn{1}{c|}{\textbf{78.17(\textcolor{blue}{2.48})}} & \multicolumn{1}{c}{51.24(\textcolor{blue}{29.41})} 
\\
\multicolumn{1}{c|}{} & \multicolumn{1}{c|}{AvgGap} & \multicolumn{1}{c|}{$-$} & \multicolumn{1}{c}{\textcolor{blue}{11.49}} & \multicolumn{1}{c}{\textcolor{blue}{12.66}} & \multicolumn{1}{c}{\textcolor{blue}{12.10}} & \multicolumn{1}{c|}{\textcolor{blue}{\textbf{4.64}}} & \multicolumn{1}{c}{\textcolor{blue}{7.65}}
\\ \bottomrule[1.5pt]
\end{tabular}

%% file: Tiny_mean.tex
\begin{tabular}{ccrrrrrr}
\toprule[1.5pt]\vspace{-0.2em}
\multirow{2.2}{*}{\begin{tabular}[c]{@{}c@{}}Unlearning\\ scenario\end{tabular}} & \multirow{2.2    }{*}{Metric} & \multicolumn{6}{c}{Methods}{} \\ \cmidrule[0.75pt]{3-8} 
 &  & \multicolumn{1}{c|}{Retrain} & \multicolumn{1}{c}{Amnesiac} & \multicolumn{1}{c}{BadTeacher} & \multicolumn{1}{c}{SalUn} & \multicolumn{1}{c|}{SSD} & \multicolumn{1}{c}{MUSO} \\ \midrule[1pt]

\multicolumn{1}{c|}{\multirow{5}{*}{\begin{tabular}[c]{@{}c@{}}Random\\ (1\%)\end{tabular}}} & \multicolumn{1}{c|}{RA} & \multicolumn{1}{c|}{\text{99.98}} & \multicolumn{1}{c}{99.98(\textcolor{blue}{0.00})} & \multicolumn{1}{c}{\textbf{99.98(\textcolor{blue}{0.00})}} & \multicolumn{1}{c}{99.98(\textcolor{blue}{0.00})} & \multicolumn{1}{c|}{96.75(\textcolor{blue}{3.23})} & \multicolumn{1}{c}{99.97(\textcolor{blue}{0.01})}
\\
\multicolumn{1}{c|}{} & \multicolumn{1}{c|}{TA} & \multicolumn{1}{c|}{\text{66.20}} & \multicolumn{1}{c}{64.18(\textcolor{blue}{2.02})} & \multicolumn{1}{c}{\textbf{64.43(\textcolor{blue}{1.77})}} & \multicolumn{1}{c}{64.15(\textcolor{blue}{2.05})} & \multicolumn{1}{c|}{61.33(\textcolor{blue}{4.87})} & \multicolumn{1}{c}{64.33(\textcolor{blue}{1.87})} 
\\
\multicolumn{1}{c|}{} & \multicolumn{1}{c|}{FA} & \multicolumn{1}{c|}{66.63} & \multicolumn{1}{c}{64.50(\textcolor{blue}{2.13})} & \multicolumn{1}{c}{66.40(\textcolor{blue}{0.23})} & \multicolumn{1}{c}{\textbf{66.55(\textcolor{blue}{0.08})}} & \multicolumn{1}{c|}{96.30(\textcolor{blue}{29.67})} & \multicolumn{1}{c}{66.85(\textcolor{blue}{0.22})}
\\
\multicolumn{1}{c|}{} & \multicolumn{1}{c|}{MIA} & \multicolumn{1}{c|}{42.63} & \multicolumn{1}{c}{1.50(\textcolor{blue}{41.13})} & \multicolumn{1}{c}{0.45(\textcolor{blue}{42.18})} & \multicolumn{1}{c}{2.65(\textcolor{blue}{39.98})} & \multicolumn{1}{c|}{87.80(\textcolor{blue}{45.17})} & \multicolumn{1}{c}{\textbf{9.35(\textcolor{blue}{33.28})}} 
\\
\multicolumn{1}{c|}{} & \multicolumn{1}{c|}{AvgGap} & \multicolumn{1}{c|}{$-$} & \multicolumn{1}{c}{\textcolor{blue}{11.32}} & \multicolumn{1}{c}{\textcolor{blue}{11.04}} & \multicolumn{1}{c}{\textcolor{blue}{10.53}} & \multicolumn{1}{c|}{\textcolor{blue}{20.74}} & \multicolumn{1}{c}{\textcolor{blue}{\textbf{8.84}}} 
\\ \midrule
\multicolumn{1}{c|}{\multirow{5}{*}{\begin{tabular}[c]{@{}c@{}}Random\\ (10\%)\end{tabular}}} & \multicolumn{1}{c|}{RA} & \multicolumn{1}{c|}{\text{99.97}} & \multicolumn{1}{c}{99.98(\textcolor{blue}{0.01})} & \multicolumn{1}{c}{\textbf{99.97(\textcolor{blue}{0.00})}} & \multicolumn{1}{c}{99.72(\textcolor{blue}{0.25})} & \multicolumn{1}{c|}{91.21(\textcolor{blue}{8.76})} & \multicolumn{1}{c}{99.47(\textcolor{blue}{0.50})} 
\\
\multicolumn{1}{c|}{} & \multicolumn{1}{c|}{TA} & \multicolumn{1}{c|}{\text{64.38}} & \multicolumn{1}{c}{62.52(\textcolor{blue}{1.86})} & \multicolumn{1}{c}{\textbf{62.61(\textcolor{blue}{1.77})}} & \multicolumn{1}{c}{60.51(\textcolor{blue}{3.87})} & \multicolumn{1}{c|}{59.91(\textcolor{blue}{4.47})} & \multicolumn{1}{c}{61.42(\textcolor{blue}{2.96})} 
\\
\multicolumn{1}{c|}{} & \multicolumn{1}{c|}{FA} & \multicolumn{1}{c|}{64.69} & \multicolumn{1}{c}{63.20(\textcolor{blue}{1.49})} & \multicolumn{1}{c}{66.32(\textcolor{blue}{1.63})} & \multicolumn{1}{c}{\textbf{65.17(\textcolor{blue}{0.48})}} & \multicolumn{1}{c|}{91.16(\textcolor{blue}{26.47})} & \multicolumn{1}{c}{66.42(\textcolor{blue}{1.73})} 
\\
\multicolumn{1}{c|}{} & \multicolumn{1}{c|}{MIA} & \multicolumn{1}{c|}{38.51} & \multicolumn{1}{c}{1.40(\textcolor{blue}{37.11})} & \multicolumn{1}{c}{0.00(\textcolor{blue}{38.51})} & \multicolumn{1}{c}{3.59(\textcolor{blue}{34.92})} & \multicolumn{1}{c|}{88.87(\textcolor{blue}{50.36})} & \multicolumn{1}{c}{\textbf{11.63(\textcolor{blue}{26.88})}} 
\\
\multicolumn{1}{c|}{} & \multicolumn{1}{c|}{AvgGap} & \multicolumn{1}{c|}{$-$} & \multicolumn{1}{c}{\textcolor{blue}{10.12}} & \multicolumn{1}{c}{\textcolor{blue}{10.47}} & \multicolumn{1}{c}{\textcolor{blue}{9.88}} & \multicolumn{1}{c|}{\textcolor{blue}{22.52}} & \multicolumn{1}{c}{\textcolor{blue}{\textbf{8.01}}}
\\ \bottomrule[1.5pt]
\end{tabular}

%% file: linear_std.tex
\begin{tabular}{ccccccc}
\toprule[1.5pt]\vspace{-0.2em}
\multirow{2}{*}{\begin{tabular}[c]{@{}c@{}}Unlearning \\ scenario\end{tabular}} & \multirow{2}{*}{Metric} & \multicolumn{5}{c}{Methods} \\ \cmidrule[0.75pt]{3-7}  
 &  & \multicolumn{1}{c|}{Retrain} & \multicolumn{1}{c}{Pre-trained} & \multicolumn{1}{c}{Amnesiac} & \multicolumn{1}{c|}{BadTeacher} & \multicolumn{1}{c}{MUSO} \\ \midrule[1pt]
 \multicolumn{1}{c|}{\multirow{4}{*}{Full-class}} & \multicolumn{1}{c|}{RA} & \multicolumn{1}{c|}{100.00$\pm$\text{0.00}} & \multicolumn{1}{c}{100.00$\pm$\text{0.00}} & \multicolumn{1}{c}{100.00$\pm$\text{0.00}} & \multicolumn{1}{c|}{100.00$\pm$\text{0.00}} & \multicolumn{1}{c}{\textbf{100.00$\pm$\text{0.00}}} \\
 \multicolumn{1}{c|}{} & \multicolumn{1}{c|}{TA} & \multicolumn{1}{c|}{50.77$\pm$\text{1.80}} & \multicolumn{1}{c}{94.42$\pm$\text{0.70}} & \multicolumn{1}{c}{56.00$\pm$\text{2.16}} & \multicolumn{1}{c|}{56.38$\pm$\text{2.28}} & \multicolumn{1}{c}{\textbf{50.87$\pm$\text{1.89}}} \\
\multicolumn{1}{c|}{} & \multicolumn{1}{c|}{FA} & \multicolumn{1}{c|}{5.33$\pm$\text{3.97}} & \multicolumn{1}{c}{100.00$\pm$\text{0.00}} & \multicolumn{1}{c}{6.53$\pm$\text{2.24}} & \multicolumn{1}{c|}{5.83$\pm$\text{4.20}} & \multicolumn{1}{c}{\textbf{5.39$\pm$\text{4.17}}} \\
\multicolumn{1}{c|}{} & \multicolumn{1}{c|}{$\Delta_{\|\vw\|}$} & \multicolumn{1}{c|}{$-$} & \multicolumn{1}{c}{{0.1027$\pm$\text{0.0029}}} & \multicolumn{1}{c}{{0.1037$\pm$\text{0.0038}}} & \multicolumn{1}{c|}{{0.1039$\pm$\text{0.0044}}} & \multicolumn{1}{c}{\textbf{{0.0001$\pm$\text{0.001}}}} \\ \midrule
\multicolumn{1}{c|}{\multirow{4}{*}{Sub-class}} & \multicolumn{1}{c|}{RA} & \multicolumn{1}{c|}{100.00$\pm$\text{0.00}} & \multicolumn{1}{c}{100.00$\pm$\text{0.00}} & \multicolumn{1}{c}{100.00$\pm$\text{0.00}} & \multicolumn{1}{c|}{100.00$\pm$\text{0.00}} & \multicolumn{1}{c}{\textbf{100.00$\pm$\text{0.00}}} \\
\multicolumn{1}{c|}{} & \multicolumn{1}{c|}{TA} & \multicolumn{1}{c|}{92.10$\pm$\text{0.78}} & \multicolumn{1}{c}{94.58$\pm$\text{0.61}} & \multicolumn{1}{c}{84.63$\pm$\text{2.27}} & \multicolumn{1}{c|}{83.25$\pm$\text{3.37}} & \multicolumn{1}{c}{\textbf{92.05$\pm$\text{0.74}}} \\
\multicolumn{1}{c|}{} & \multicolumn{1}{c|}{FA} & \multicolumn{1}{c|}{87.60$\pm$\text{1.69}} & \multicolumn{1}{c}{100.00$\pm$\text{0.00}} & \multicolumn{1}{c}{88.20$\pm$\text{1.50}} & \multicolumn{1}{c|}{\textbf{87.60$\pm$\text{1.74}}} & \multicolumn{1}{c}{87.30$\pm$\text{2.25}} \\
\multicolumn{1}{c|}{} & \multicolumn{1}{c|}{$\Delta_{\|\vw\|}$} & \multicolumn{1}{c|}{$-$} & \multicolumn{1}{c}{{0.0554$\pm$\text{0.0014}}} & \multicolumn{1}{c}{{0.0566$\pm$\text{0.0015}}} & \multicolumn{1}{c|}{{0.0577$\pm$\text{0.0042}}} & \multicolumn{1}{c}{\textbf{{0.0001$\pm$\text{0.0000}}}} \\ \midrule
\multicolumn{1}{c|}{\multirow{4}{*}{Random}} & \multicolumn{1}{c|}{RA} & \multicolumn{1}{c|}{100.00$\pm$\text{0.00}} & \multicolumn{1}{c}{100.00$\pm$\text{0.00}} & \multicolumn{1}{c}{96.80$\pm$\text{2.84}} & \multicolumn{1}{c|}{95.00$\pm$\text{4.02}} & \multicolumn{1}{c}{\textbf{100.00$\pm$\text{0.00}}} \\
\multicolumn{1}{c|}{} & \multicolumn{1}{c|}{TA} & \multicolumn{1}{c|}{92.63$\pm$\text{0.71}} & \multicolumn{1}{c}{94.06$\pm$\text{0.65}} & \multicolumn{1}{c}{84.75$\pm$\text{1.59}} & \multicolumn{1}{c|}{80.97$\pm$\text{4.61}} & \multicolumn{1}{c}{\textbf{92.74$\pm$\text{0.70}}} \\
\multicolumn{1}{c|}{} & \multicolumn{1}{c|}{FA} & \multicolumn{1}{c|}{93.90$\pm$\text{0.92}} & \multicolumn{1}{c}{100.00$\pm$\text{0.00}} & \multicolumn{1}{c}{94.20$\pm$\text{1.86}} & \multicolumn{1}{c|}{93.40$\pm$\text{1.50}} & \multicolumn{1}{c}{\textbf{93.90$\pm$\text{0.92}}} \\
\multicolumn{1}{c|}{} & \multicolumn{1}{c|}{$\Delta_{\|\vw\|}$} & \multicolumn{1}{c|}{$-$} & \multicolumn{1}{c}{{0.0485$\pm$\text{0.0034}}} & \multicolumn{1}{c}{{0.0501$\pm$\text{0.0043}}} & \multicolumn{1}{c|}{{0.0506$\pm$\text{0.0036}}} & \multicolumn{1}{c}{\textbf{{0.0002$\pm$\text{0.0001}}}} \\ \bottomrule[1.5pt]
\end{tabular}

%% file: NN_std.tex
\begin{tabular}{ccrrrrrr}
\toprule[1.5pt]\vspace{-0.2em}
\multirow{2}{*}{\begin{tabular}[c]{@{}c@{}}Forgetting\\ scenario\end{tabular}} & \multirow{2}{*}{Metric} & \multicolumn{6}{c}{Methods}{} \\ \cmidrule[0.75pt]{3-8} 
 &  & \multicolumn{1}{c|}{Retrain} & \multicolumn{1}{c}{Amnesiac} & \multicolumn{1}{c}{BadTeacher} & \multicolumn{1}{c}{SalUn} & \multicolumn{1}{c|}{SSD} & \multicolumn{1}{c}{MUSO} \\ \midrule[1pt]
\multicolumn{1}{c|}{\multirow{4}{*}{\begin{tabular}[c]{@{}c@{}}Full-class\\ (Rocket)\end{tabular}}} & \multicolumn{1}{c|}{RA} & \multicolumn{1}{c|}{\text{99.96}$\pm$\text{0.01}} & \multicolumn{1}{c}{99.97$\pm$\text{0.00}} & \multicolumn{1}{c}{99.93$\pm$\text{0.01}} & \multicolumn{1}{c}{99.97$\pm$\text{0.00}} & \multicolumn{1}{c|}{99.08$\pm$\text{0.04}} & \multicolumn{1}{c}{\textbf{99.97$\pm$\text{0.00}}} \\
\multicolumn{1}{c|}{} & \multicolumn{1}{c|}{TA} & \multicolumn{1}{c|}{\text{76.26}$\pm$\text{0.23}} & \multicolumn{1}{c}{76.72$\pm$\text{0.16}} & \multicolumn{1}{c}{76.69$\pm$\text{0.12}} & \multicolumn{1}{c}{76.66$\pm$\text{0.14}} & \multicolumn{1}{c|}{76.52$\pm$\text{0.07}} & \multicolumn{1}{c}{\textbf{76.23$\pm$\text{0.35}}} \\
\multicolumn{1}{c|}{} & \multicolumn{1}{c|}{FA} & \multicolumn{1}{c|}{0.00$\pm$\text{0.00}} & \multicolumn{1}{c}{0.00$\pm$\text{0.00}} & \multicolumn{1}{c}{1.20$\pm$\text{0.45}} & \multicolumn{1}{c}{0.00$\pm$\text{0.00}} & \multicolumn{1}{c|}{0.00$\pm$\text{0.00}} & \multicolumn{1}{c}{\textbf{0.00$\pm$\text{0.00}}} \\
\multicolumn{1}{c|}{} & \multicolumn{1}{c|}{MIA} & \multicolumn{1}{c|}{15.96$\pm$\text{1.69}} & \multicolumn{1}{c}{9.08$\pm$\text{1.04}} & \multicolumn{1}{c}{0.00$\pm$\text{0.00}} & \multicolumn{1}{c}{5.28$\pm$\text{0.78}} & \multicolumn{1}{c|}{1.36$\pm$\text{0.17}} & \multicolumn{1}{c}{\textbf{20.24$\pm$\text{1.19}}} \\ \midrule
\multicolumn{1}{c|}{\multirow{4}{*}{\begin{tabular}[c]{@{}c@{}}Full-class\\ (Sea)\end{tabular}}} & \multicolumn{1}{c|}{RA} & \multicolumn{1}{c|}{\text{99.94}$\pm$\text{0.03}} & \multicolumn{1}{c}{99.97$\pm$\text{0.00}} & \multicolumn{1}{c}{\textbf{99.94$\pm$\text{0.01}}} & \multicolumn{1}{c}{99.97$\pm$\text{0.00}} & \multicolumn{1}{c|}{98.91$\pm$\text{0.32}} & \multicolumn{1}{c}{99.96$\pm$\text{0.01}} \\
\multicolumn{1}{c|}{} & \multicolumn{1}{c|}{TA} & \multicolumn{1}{c|}{\text{76.53}$\pm$\text{0.13}} & \multicolumn{1}{c}{76.17$\pm$\text{0.25}} & \multicolumn{1}{c}{77.65$\pm$\text{0.14}} & \multicolumn{1}{c}{\textbf{76.46$\pm$\text{0.36}}} & \multicolumn{1}{c|}{74.54$\pm$\text{0.28}} & \multicolumn{1}{c}{76.44$\pm$\text{0.27}} \\
\multicolumn{1}{c|}{} & \multicolumn{1}{c|}{FA} & \multicolumn{1}{c|}{0.00$\pm$\text{0.00}} & \multicolumn{1}{c}{0.00$\pm$\text{0.00}} & \multicolumn{1}{c}{32.20$\pm$\text{5.26}} & \multicolumn{1}{c}{0.00$\pm$\text{0.00}} & \multicolumn{1}{c|}{0.00$\pm$\text{0.00}} & \multicolumn{1}{c}{\textbf{0.00$\pm$\text{0.00}}} \\
\multicolumn{1}{c|}{} & \multicolumn{1}{c|}{MIA} & \multicolumn{1}{c|}{25.64$\pm$\text{2.12}} & \multicolumn{1}{c}{5.88$\pm$\text{0.76}} & \multicolumn{1}{c}{0.00$\pm$\text{0.00}} & \multicolumn{1}{c}{5.88$\pm$\text{1.01}} & \multicolumn{1}{c|}{1.64$\pm$\text{0.30}} & \multicolumn{1}{c}{\textbf{24.45$\pm$\text{4.27}}} \\ \midrule
\multicolumn{1}{c|}{\multirow{4}{*}{\begin{tabular}[c]{@{}c@{}}Full-class\\ (Cattle)\end{tabular}}} & \multicolumn{1}{c|}{RA} & \multicolumn{1}{c|}{\text{99.93}$\pm$\text{0.03}} & \multicolumn{1}{c}{99.98$\pm$\text{0.00}} & \multicolumn{1}{c}{\textbf{99.94$\pm$\text{0.03}}} & \multicolumn{1}{c}{99.97$\pm$\text{0.00}} & \multicolumn{1}{c|}{99.07$\pm$\text{0.01}} & \multicolumn{1}{c}{99.97$\pm$\text{0.00}} \\
\multicolumn{1}{c|}{} & \multicolumn{1}{c|}{TA} & \multicolumn{1}{c|}{\text{76.69}$\pm$\text{0.15}} & \multicolumn{1}{c}{76.97$\pm$\text{0.13}} & \multicolumn{1}{c}{77.70$\pm$\text{0.07}} & \multicolumn{1}{c}{76.60$\pm$\text{0.19}} & \multicolumn{1}{c|}{76.63$\pm$\text{0.03}} & \multicolumn{1}{c}{\textbf{76.68$\pm$\text{0.12}}} \\
\multicolumn{1}{c|}{} & \multicolumn{1}{c|}{FA} & \multicolumn{1}{c|}{0.00$\pm$\text{0.00}} & \multicolumn{1}{c}{0.00$\pm$\text{0.00}} & \multicolumn{1}{c}{6.80$\pm$\text{3.03}} & \multicolumn{1}{c}{0.00$\pm$\text{0.00}} & \multicolumn{1}{c|}{0.00$\pm$\text{0.00}} & \multicolumn{1}{c}{\textbf{0.00$\pm$\text{0.00}}} \\
\multicolumn{1}{c|}{} & \multicolumn{1}{c|}{MIA} & \multicolumn{1}{c|}{11.48$\pm$\text{1.44}} & \multicolumn{1}{c}{10.20$\pm$\text{0.84}} & \multicolumn{1}{c}{0.00$\pm$\text{0.00}} & \multicolumn{1}{c}{\textbf{10.60$\pm$\text{0.35}}} & \multicolumn{1}{c|}{2.48$\pm$\text{0.18}} & \multicolumn{1}{c}{8.67$\pm$\text{1.10}} \\ \midrule[1.5pt]
\multicolumn{1}{c|}{\multirow{4}{*}{\begin{tabular}[c]{@{}c@{}}Sub-class\\ (Rocket)\end{tabular}}} & \multicolumn{1}{c|}{RA} & \multicolumn{1}{c|}{\text{99.96}$\pm$\text{0.01}} & \multicolumn{1}{c}{\textbf{99.97$\pm$\text{0.00}}} & \multicolumn{1}{c}{99.91$\pm$\text{0.02}} & \multicolumn{1}{c}{99.93$\pm$\text{0.01}} & \multicolumn{1}{c|}{99.29$\pm$\text{0.15}} & \multicolumn{1}{c}{99.99$\pm$\text{0.00}} \\
\multicolumn{1}{c|}{} & \multicolumn{1}{c|}{TA} & \multicolumn{1}{c|}{\text{84.95}$\pm$\text{0.24}} & \multicolumn{1}{c}{84.55$\pm$\text{0.07}} & \multicolumn{1}{c}{\textbf{84.99$\pm$\text{0.13}}} & \multicolumn{1}{c}{84.58$\pm$\text{0.06}} & \multicolumn{1}{c|}{84.74$\pm$\text{0.06}} & \multicolumn{1}{c}{84.46$\pm$\text{0.43}} \\
\multicolumn{1}{c|}{} & \multicolumn{1}{c|}{FA} & \multicolumn{1}{c|}{2.60$\pm$\text{0.55}} & \multicolumn{1}{c}{\textbf{2.60$\pm$\text{0.55}}} & \multicolumn{1}{c}{6.20$\pm$\text{2.17}} & \multicolumn{1}{c}{3.20$\pm$\text{0.84}} & \multicolumn{1}{c|}{4.00$\pm$\text{2.12}} & \multicolumn{1}{c}{1.00$\pm$\text{0.00}} \\
\multicolumn{1}{c|}{} & \multicolumn{1}{c|}{MIA} & \multicolumn{1}{c|}{21.36$\pm$\text{2.05}} & \multicolumn{1}{c}{0.00$\pm$\text{0.00}} & \multicolumn{1}{c}{0.00$\pm$\text{0.00}} & \multicolumn{1}{c}{0.00$\pm$\text{0.00}} & \multicolumn{1}{c|}{5.28$\pm$\text{0.41}} & \multicolumn{1}{c}{\textbf{20.84$\pm$\text{4.20}}} \\ \midrule
\multicolumn{1}{c|}{\multirow{4}{*}{\begin{tabular}[c]{@{}c@{}}Sub-class\\ (Sea)\end{tabular}}} & \multicolumn{1}{c|}{RA} & \multicolumn{1}{c|}{\text{99.93}$\pm$\text{0.03}} & \multicolumn{1}{c}{\textbf{99.97$\pm$\text{0.00}}} & \multicolumn{1}{c}{99.37$\pm$\text{0.02}} & \multicolumn{1}{c}{99.81$\pm$\text{0.01}} & \multicolumn{1}{c|}{98.17$\pm$\text{0.28}} & \multicolumn{1}{c}{99.60$\pm$\text{0.15}} \\
\multicolumn{1}{c|}{} & \multicolumn{1}{c|}{TA} & \multicolumn{1}{c|}{\text{84.74}$\pm$\text{0.25}} & \multicolumn{1}{c}{86.37$\pm$\text{4.45}} & \multicolumn{1}{c}{\textbf{84.82$\pm$\text{0.14}}} & \multicolumn{1}{c}{84.46$\pm$\text{0.05}} & \multicolumn{1}{c|}{84.51$\pm$\text{0.06}} & \multicolumn{1}{c}{82.62$\pm$\text{0.14}} \\
\multicolumn{1}{c|}{} & \multicolumn{1}{c|}{FA} & \multicolumn{1}{c|}{82.40$\pm$\text{3.36}} & \multicolumn{1}{c}{78.00$\pm$\text{2.35}} & \multicolumn{1}{c}{\textbf{78.00$\pm$\text{1.58}}} & \multicolumn{1}{c}{79.20$\pm$\text{0.84}} & \multicolumn{1}{c|}{64.60$\pm$\text{2.07}} & \multicolumn{1}{c}{88.80$\pm$\text{1.10}} \\
\multicolumn{1}{c|}{} & \multicolumn{1}{c|}{MIA} & \multicolumn{1}{c|}{59.08$\pm$\text{1.25}} & \multicolumn{1}{c}{0.04$\pm$\text{0.09}} & \multicolumn{1}{c}{0.16$\pm$\text{0.09}} & \multicolumn{1}{c}{0.00$\pm$\text{0.00}} & \multicolumn{1}{c|}{7.48$\pm$\text{2.24}} & \multicolumn{1}{c}{\textbf{55.88$\pm$\text{2.11}}} \\ \midrule
\multicolumn{1}{c|}{\multirow{4}{*}{\begin{tabular}[c]{@{}c@{}}Sub-class\\ (Cattle)\end{tabular}}} & \multicolumn{1}{c|}{RA} & \multicolumn{1}{c|}{\text{99.93}$\pm$\text{0.04}} & \multicolumn{1}{c}{99.97$\pm$\text{0.00}} & \multicolumn{1}{c}{99.39$\pm$\text{0.01}} & \multicolumn{1}{c}{\textbf{99.93$\pm$\text{0.01}}} & \multicolumn{1}{c|}{99.06$\pm$\text{0.45}} & \multicolumn{1}{c}{99.92$\pm$\text{0.02}} \\
\multicolumn{1}{c|}{} & \multicolumn{1}{c|}{TA} & \multicolumn{1}{c|}{\text{84.83}$\pm$\text{0.18}} & \multicolumn{1}{c}{84.54$\pm$\text{0.06}} & \multicolumn{1}{c}{85.01$\pm$\text{0.12}} & \multicolumn{1}{c}{\textbf{84.72$\pm$\text{0.03}}} & \multicolumn{1}{c|}{83.51$\pm$\text{0.36}} & \multicolumn{1}{c}{83.95$\pm$\text{0.08}} \\
\multicolumn{1}{c|}{} & \multicolumn{1}{c|}{FA} & \multicolumn{1}{c|}{47.80$\pm$\text{1.79}} & \multicolumn{1}{c}{38.40$\pm$\text{4.39}} & \multicolumn{1}{c}{42.80$\pm$\text{2.86}} & \multicolumn{1}{c}{33.80$\pm$\text{1.92}} & \multicolumn{1}{c|}{16.40$\pm$\text{7.47}} & \multicolumn{1}{c}{\textbf{52.40$\pm$\text{1.67}}} \\
\multicolumn{1}{c|}{} & \multicolumn{1}{c|}{MIA} & \multicolumn{1}{c|}{26.64$\pm$\text{1.19}} & \multicolumn{1}{c}{0.00$\pm$\text{0.00}} & \multicolumn{1}{c}{0.00$\pm$\text{0.00}} & \multicolumn{1}{c}{0.12$\pm$\text{0.11}} & \multicolumn{1}{c|}{8.80$\pm$\text{2.22}} & \multicolumn{1}{c}{\textbf{25.14$\pm$\text{2.51}}} \\ \midrule[1.5pt]
\multicolumn{1}{c|}{\multirow{4}{*}{\begin{tabular}[c]{@{}c@{}}Random\\ (1\%)\end{tabular}}} & \multicolumn{1}{c|}{RA} & \multicolumn{1}{c|}{\text{99.99}$\pm$\text{0.01}} & \multicolumn{1}{c}{99.92$\pm$\text{0.01}} & \multicolumn{1}{c}{99.88$\pm$\text{0.01}} & \multicolumn{1}{c}{\textbf{99.93$\pm$\text{0.00}}} & \multicolumn{1}{c|}{98.16$\pm$\text{0.00}} & \multicolumn{1}{c}{99.93$\pm$\text{0.01}} \\
\multicolumn{1}{c|}{} & \multicolumn{1}{c|}{TA} & \multicolumn{1}{c|}{\text{76.41}$\pm$\text{0.31}} & \multicolumn{1}{c}{76.00$\pm$\text{0.10}} & \multicolumn{1}{c}{75.77$\pm$\text{0.17}} & \multicolumn{1}{c}{\textbf{76.16$\pm$\text{0.06}}} & \multicolumn{1}{c|}{72.48$\pm$\text{1.03}} & \multicolumn{1}{c}{75.93$\pm$\text{0.18}} \\
\multicolumn{1}{c|}{} & \multicolumn{1}{c|}{FA} & \multicolumn{1}{c|}{75.28$\pm$\text{0.33}} & \multicolumn{1}{c}{\textbf{75.68$\pm$\text{0.54}}} & \multicolumn{1}{c}{76.60$\pm$\text{1.48}} & \multicolumn{1}{c}{82.16$\pm$\text{0.77}} & \multicolumn{1}{c|}{93.40$\pm$\text{1.50}} & \multicolumn{1}{c}{76.20$\pm$\text{1.36}} \\
\multicolumn{1}{c|}{} & \multicolumn{1}{c|}{MIA} & \multicolumn{1}{c|}{55.44$\pm$\text{0.33}} & \multicolumn{1}{c}{10.28$\pm$\text{0.39}} & \multicolumn{1}{c}{10.80$\pm$\text{0.20}} & \multicolumn{1}{c}{18.44$\pm$\text{0.61}} & \multicolumn{1}{c|}{\textbf{79.16$\pm$\text{2.05}}} & \multicolumn{1}{c}{19.72$\pm$\text{1.21}} \\ \midrule
\multicolumn{1}{c|}{\multirow{4}{*}{\begin{tabular}[c]{@{}c@{}}Random\\ (10\%)\end{tabular}}} & \multicolumn{1}{c|}{RA} & \multicolumn{1}{c|}{\text{99.99}$\pm$\text{0.00}} & \multicolumn{1}{c}{\textbf{99.89$\pm$\text{0.00}}} & \multicolumn{1}{c}{99.86$\pm$\text{0.01}} & \multicolumn{1}{c}{99.83$\pm$\text{0.01}} & \multicolumn{1}{c|}{96.04$\pm$\text{0.97}} & \multicolumn{1}{c}{99.85$\pm$\text{0.04}} \\
\multicolumn{1}{c|}{} & \multicolumn{1}{c|}{TA} & \multicolumn{1}{c|}{\text{75.44}$\pm$\text{0.05}} & \multicolumn{1}{c}{72.70$\pm$\text{0.39}} & \multicolumn{1}{c}{\textbf{74.44$\pm$\text{0.15}}} & \multicolumn{1}{c}{72.28$\pm$\text{0.15}} & \multicolumn{1}{c|}{72.59$\pm$\text{0.49}} & \multicolumn{1}{c}{70.04$\pm$\text{0.31}} \\
\multicolumn{1}{c|}{} & \multicolumn{1}{c|}{FA} & \multicolumn{1}{c|}{74.68$\pm$\text{0.19}} & \multicolumn{1}{c}{75.95$\pm$\text{0.40}} & \multicolumn{1}{c}{\textbf{74.86$\pm$\text{0.91}}} & \multicolumn{1}{c}{76.92$\pm$\text{0.36}} & \multicolumn{1}{c|}{92.88$\pm$\text{0.89}} & \multicolumn{1}{c}{73.09$\pm$\text{2.54}} \\
\multicolumn{1}{c|}{} & \multicolumn{1}{c|}{MIA} & \multicolumn{1}{c|}{54.59$\pm$\text{0.30}} & \multicolumn{1}{c}{4.02$\pm$\text{0.14}} & \multicolumn{1}{c}{0.08$\pm$\text{0.03}} & \multicolumn{1}{c}{2.53$\pm$\text{0.19}} & \multicolumn{1}{c|}{\textbf{81.41$\pm$\text{0.28}}} & \multicolumn{1}{c}{18.32$\pm$\text{2.01}} \\

\bottomrule[1.5pt]
\end{tabular}

%% file: VGG_std.tex
\begin{tabular}{ccrrrrrr}
\toprule[1.5pt]\vspace{-0.2em}
\multirow{2}{*}{\begin{tabular}[c]{@{}c@{}}Unlearning\\ scenario\end{tabular}} & \multirow{2}{*}{Metric} & \multicolumn{6}{c}{Methods}{} \\ \cmidrule[0.75pt]{3-8} 
 &  & \multicolumn{1}{c|}{Retrain} & \multicolumn{1}{c}{Amnesiac} & \multicolumn{1}{c}{BadTeacher} & \multicolumn{1}{c}{SalUn} & \multicolumn{1}{c|}{SSD} & \multicolumn{1}{c}{MUSO} \\ \midrule[1pt]

\multicolumn{1}{c|}{\multirow{4}{*}{\begin{tabular}[c]{@{}c@{}}Random\\ (1\%)\end{tabular}}} & \multicolumn{1}{c|}{RA} & \multicolumn{1}{c|}{\text{99.97$\pm$0.01}} & \multicolumn{1}{c}{100.00$\pm$0.00} & \multicolumn{1}{c}{99.95$\pm$0.01} & \multicolumn{1}{c}{100.00$\pm$0.00} & \multicolumn{1}{c|}{93.37$\pm$0.47} & \multicolumn{1}{c}{\textbf{99.99$\pm$0.01}} 
\\
\multicolumn{1}{c|}{} & \multicolumn{1}{c|}{TA} & \multicolumn{1}{c|}{93.38$\pm$0.08} & \multicolumn{1}{c}{93.00$\pm$0.17} & \multicolumn{1}{c}{\textbf{93.12$\pm$0.12}} & \multicolumn{1}{c}{92.71$\pm$0.10} & \multicolumn{1}{c|}{84.05$\pm$0.54} & \multicolumn{1}{c}{92.73$\pm$0.05} 
\\
\multicolumn{1}{c|}{} & \multicolumn{1}{c|}{FA} & \multicolumn{1}{c|}{92.50$\pm$0.81} & \multicolumn{1}{c}{89.96$\pm$0.94} & \multicolumn{1}{c}{92.36$\pm$1.23} & \multicolumn{1}{c}{92.13$\pm$0.25} & \multicolumn{1}{c|}{\textbf{92.60$\pm$0.63}} & \multicolumn{1}{c}{92.20$\pm$0.85} 
\\
\multicolumn{1}{c|}{} & \multicolumn{1}{c|}{MIA} & \multicolumn{1}{c|}{79.55$\pm$0.64} & \multicolumn{1}{c}{41.76$\pm$1.34} & \multicolumn{1}{c}{52.80$\pm$2.49} & \multicolumn{1}{c}{47.15$\pm$0.74} & \multicolumn{1}{c|}{\textbf{77.50$\pm$0.71}} & \multicolumn{1}{c}{55.00$\pm$1.41} 
\\ \midrule

\multicolumn{1}{c|}{\multirow{4}{*}{\begin{tabular}[c]{@{}c@{}}Random\\ (10\%)\end{tabular}}} & \multicolumn{1}{c|}{RA} & \multicolumn{1}{c|}{99.96$\pm$0.01} & \multicolumn{1}{c}{100.00$\pm$0.00} & \multicolumn{1}{c}{\textbf{99.97$\pm$0.01}} & \multicolumn{1}{c}{99.99$\pm$0.01} & \multicolumn{1}{c|}{92.83$\pm$1.04} & \multicolumn{1}{c}{99.98$\pm$0.00} 
\\
\multicolumn{1}{c|}{} & \multicolumn{1}{c|}{TA} & \multicolumn{1}{c|}{92.89$\pm$0.22} & \multicolumn{1}{c}{\textbf{92.90$\pm$0.07}} & \multicolumn{1}{c}{92.91$\pm$0.09} & \multicolumn{1}{c}{92.62$\pm$0.12} & \multicolumn{1}{c|}{85.42$\pm$0.36} & \multicolumn{1}{c}{91.76$\pm$0.13} 
\\
\multicolumn{1}{c|}{} & \multicolumn{1}{c|}{FA} & \multicolumn{1}{c|}{93.41$\pm$0.27} & \multicolumn{1}{c}{93.72$\pm$0.23} & \multicolumn{1}{c}{92.77$\pm$0.19} & \multicolumn{1}{c}{93.59$\pm$0.24} & \multicolumn{1}{c|}{94.89$\pm$0.78} & \multicolumn{1}{c}{\textbf{93.35$\pm$0.07}} 
\\
\multicolumn{1}{c|}{} & \multicolumn{1}{c|}{MIA} & \multicolumn{1}{c|}{80.65$\pm$0.37} & \multicolumn{1}{c}{35.06$\pm$0.41} & \multicolumn{1}{c}{30.69$\pm$0.46} & \multicolumn{1}{c}{32.73$\pm$0.66} & \multicolumn{1}{c|}{\textbf{78.17$\pm$1.24}} & \multicolumn{1}{c}{51.24$\pm$1.58}
\\ \bottomrule[1.5pt]
\end{tabular}

%% file: Tiny_std.tex
\begin{tabular}{ccrrrrrr}
\toprule[1.5pt]\vspace{-0.2em}
\multirow{2}{*}{\begin{tabular}[c]{@{}c@{}}Unlearning\\ scenario\end{tabular}} & \multirow{2}{*}{Metric} & \multicolumn{6}{c}{Methods}{} \\ \cmidrule[0.75pt]{3-8} 
 &  & \multicolumn{1}{c|}{Retrain} & \multicolumn{1}{c}{Amnesiac} & \multicolumn{1}{c}{BadTeacher} & \multicolumn{1}{c}{SalUn} & \multicolumn{1}{c|}{SSD} & \multicolumn{1}{c}{MUSO} \\ \midrule[1pt]

\multicolumn{1}{c|}{\multirow{4}{*}{\begin{tabular}[c]{@{}c@{}}Random\\ (1\%)\end{tabular}}} & \multicolumn{1}{c|}{RA} & \multicolumn{1}{c|}{\text{99.98$\pm$0.00}} & \multicolumn{1}{c}{99.98$\pm$0.01} & \multicolumn{1}{c}{\textbf{99.98$\pm$0.00}} & \multicolumn{1}{c}{99.98$\pm$0.01} & \multicolumn{1}{c|}{96.75$\pm$0.33} & \multicolumn{1}{c}{99.97$\pm$0.00} 
\\
\multicolumn{1}{c|}{} & \multicolumn{1}{c|}{TA} & \multicolumn{1}{c|}{66.20$\pm$0.38} & \multicolumn{1}{c}{64.18$\pm$0.08} & \multicolumn{1}{c}{\textbf{64.43$\pm$0.27}} & \multicolumn{1}{c}{64.15$\pm$0.20} & \multicolumn{1}{c|}{61.33$\pm$0.86} & \multicolumn{1}{c}{64.33$\pm$0.12} 
\\
\multicolumn{1}{c|}{} & \multicolumn{1}{c|}{FA} & \multicolumn{1}{c|}{66.63$\pm$0.65} & \multicolumn{1}{c}{64.50$\pm$1.27} & \multicolumn{1}{c}{66.40$\pm$2.55} & \multicolumn{1}{c}{\textbf{66.55$\pm$1.77}} & \multicolumn{1}{c|}{96.30$\pm$0.31} & \multicolumn{1}{c}{66.85$\pm$0.35} 
\\
\multicolumn{1}{c|}{} & \multicolumn{1}{c|}{MIA} & \multicolumn{1}{c|}{42.63$\pm$1.05} & \multicolumn{1}{c}{1.50$\pm$0.28} & \multicolumn{1}{c}{0.45$\pm$0.35} & \multicolumn{1}{c}{2.65$\pm$0.07} & \multicolumn{1}{c|}{87.80$\pm$1.56} & \multicolumn{1}{c}{\textbf{9.35$\pm$0.49}} 
\\ \midrule

\multicolumn{1}{c|}{\multirow{4}{*}{\begin{tabular}[c]{@{}c@{}}Random\\ (10\%)\end{tabular}}} & \multicolumn{1}{c|}{RA} & \multicolumn{1}{c|}{99.97$\pm$0.01} & \multicolumn{1}{c}{99.98$\pm$0.00} & \multicolumn{1}{c}{\textbf{99.97$\pm$0.01}} & \multicolumn{1}{c}{99.72$\pm$0.25} & \multicolumn{1}{c|}{91.21$\pm$1.05} & \multicolumn{1}{c}{99.47$\pm$0.12} 
\\
\multicolumn{1}{c|}{} & \multicolumn{1}{c|}{TA} & \multicolumn{1}{c|}{64.38$\pm$0.02} & \multicolumn{1}{c}{62.52$\pm$0.04} & \multicolumn{1}{c}{\textbf{62.61$\pm$0.74}} & \multicolumn{1}{c}{60.51$\pm$0.47} & \multicolumn{1}{c|}{59.91$\pm$0.41} & \multicolumn{1}{c}{61.42$\pm$1.03} 
\\
\multicolumn{1}{c|}{} & \multicolumn{1}{c|}{FA} & \multicolumn{1}{c|}{64.69$\pm$0.01} & \multicolumn{1}{c}{63.20$\pm$0.35} & \multicolumn{1}{c}{66.32$\pm$0.69} & \multicolumn{1}{c}{\textbf{65.17$\pm$0.35}} & \multicolumn{1}{c|}{91.16$\pm$0.98} & \multicolumn{1}{c}{66.42$\pm$2.20} 
\\
\multicolumn{1}{c|}{} & \multicolumn{1}{c|}{MIA} & \multicolumn{1}{c|}{38.51$\pm$0.51} & \multicolumn{1}{c}{1.40$\pm$0.04} & \multicolumn{1}{c}{0.00$\pm$0.00} & \multicolumn{1}{c}{3.59$\pm$0.13} & \multicolumn{1}{c|}{88.87$\pm$0.12} & \multicolumn{1}{c}{\textbf{11.63$\pm$1.46}} 
\\ \bottomrule[1.5pt]
\end{tabular}